\documentclass{article}

\PassOptionsToPackage{numbers, compress}{natbib}

\usepackage[main, final]{neurips_2026}

\usepackage[utf8]{inputenc} \usepackage[T1]{fontenc}    \usepackage{hyperref}       \usepackage{url}            \usepackage{booktabs}       \usepackage{amsfonts}       \usepackage{nicefrac}       \usepackage{microtype}      \usepackage{xcolor}         
\usepackage{amsmath}
\usepackage{amssymb}
\usepackage{mathtools}
\usepackage{amsthm}
\usepackage{thmtools,thm-restate}
\usepackage{multirow}
\usepackage{multicol}
\usepackage{wrapfig}
\usepackage{algorithm}
\usepackage{algpseudocode}
\usepackage{authblk}
\usepackage{graphicx}
\usepackage{subcaption}
\usepackage{dsfont}
\usepackage{bm}
\usepackage[capitalize,noabbrev]{cleveref}

\makeatletter

\newcommand{\model}{HRIL\futurelet\next@token\check@next@token}
\newcommand{\check@next@token}{\ifcat\noexpand\next@token a\ \else
    \ifx\next@token(\ \else
        \fi
  \fi
}

\newtheorem{proposition}{Proposition}
\newtheorem{lemma}{Lemma}

\newtheorem{definition}{Definition}

\newtheorem{assumption}{Assumption}

\newcommand{\vectorname}[1]{{\mathrm{\mathbf{#1}}}}
\newcommand{\ie}{\textit{i.e.}}
\newcommand{\eg}{\textit{e.g.}}

\title{\model: Learning Multimodal Synergy \\ via Higher-Order Tensor Modeling}

\author{
\textbf{Qun Dai$^{1}$, Liangjian Wen$^{1,2\ast}$, Jiang Duan$^{1,7}$, Yong Dai$^{3}$, Dongkai Wang$^{1}$, Maolin Wang$^{4}$, Mingjie Wang$^{5}$, Jianzhuang Liu$^{6}$, HE YAN$^{7}$, Zhao Kang$^{8}$} \\

$^1$\textmd{School of Computing and Artificial Intelligence,  Southwestern University of Finance and Economics}\\
$^2$Artificial Intelligence and Digital Finance Key Laboratory of Sichuan Province\\
$^3$X-Humanoid \\
$^4$Hong Kong Institute of AI for Science, City University of Hong Kong\\
$^5$Zhejiang Sci-Tech University\\
$^6$Shenzhen Institutes of Advanced Technology, Chinese Academy of Sciences \\
$^7$Chengdu Everimaging Science and Technology Co., Ltd\\
$^8$University of Electronic Science and Technology of China\\
\texttt{\{daiqun1124,wlj6816\}@gmail.com} \\
}

\begin{document}

\maketitle
\renewcommand{\thefootnote}{\fnsymbol{footnote}}
    \footnotetext[1]{Corresponding author.}

\begin{abstract}
Self-supervised multimodal representation learning has achieved remarkable success across diverse domains, 
yet capturing \emph{synergistic} information remains challenging due to the complexity of cross-modal interactions.
Unlike the shared information across individual modalities, synergy arises when task-relevant signals emerge only from the joint configuration of multiple modalities and cannot be recovered from any modality in isolation. 
This work focuses on how to preserve the information capacity for such synergistic signals in multimodal representations.
The key observation is that synergistic information is reflected in higher-order statistical dependence among modalities, which provides a principled target for explicitly modeling joint interactions. 
Motivated by this insight, we propose \textbf{H}igher-order \textbf{R}epresentation and \textbf{I}nformation \textbf{L}earning (\model), 
which constructs an empirical cross-moment tensor over modality embeddings to represent multi-way interactions.
\model employs Tucker decomposition to obtain a core tensor,
complemented by a synergy-aware regularizer that prevents energy concentration and preserves higher-order coupling capacity for synergistic information capture.
Experiments on the controlled synergy task and real-world benchmarks demonstrate consistent improvements over existing multimodal contrastive methods, with notable gains on tasks dominated by synergistic interactions.
Code is released at \url{https://github.com/brightest66/HRIL}.
\end{abstract}

\section{Introduction}
Self-supervised multimodal representation learning constructs joint representations from heterogeneous data sources whose underlying statistical structures differ fundamentally.
Since distinct modalities arise from disparate data-generating processes and reside in qualitatively different statistical spaces,
the central challenge is to leverage complementary and heterogeneous features across modalities for representations that are maximally informative for downstream tasks.
Foundational pairwise contrastive methods such as CLIP~\cite{radford2021learning} and ALIGN~\cite{jia2021scaling} have achieved remarkable transferability across vision, language, and audio domains, 
yet they operate under the \textit{multiview redundancy assumption}~\cite{tsai2020self, DBLP:conf/colt/SridharanK08,10477543}, which posits that task-relevant information is largely shared across modalities.

When we examine multimodal interactions across two or more modalities through the lens of Partial Information Decomposition (PID)~\cite{williams2010nonnegative, bertschinger2014quantifying}
--- which decomposes task-relevant information into \emph{redundancy} (shared across modalities), \emph{uniqueness} (exclusive to individual modalities), and \emph{synergy} (accessible only via joint observation of multiple modalities)
--- the limitations of redundancy-centric pairwise contrastive learning become apparent.
For instance, in the MM-IMDb dataset~\cite{arevalo2017gated}, accurate genre prediction necessitates the synergistic integration of textual plot descriptions and visual poster elements, as the genre semantics emerge from their cross-modal combination rather than being fully encapsulated in either modality alone.
Consequently, synergy is of fundamental importance in multimodal representation learning, 
as it embodies the irreducible cross-modal dependency that justifies multimodal representation learning and cannot be inferred from any single modality~\cite{dufumier2024align,wen2025infmasking}.

A natural question is whether stronger pairwise contrastive learning can capture synergy.
\cite{dufumier2024align} demonstrates that cross-modal contrastive learning under the \emph{multiview redundancy assumption} primarily captures shared information in bimodal settings.
In this work, we further prove that a representation which is \emph{sufficient} for all pairwise contrastive learning yet \emph{minimal} in information content provably discards synergistic information accessible only through the joint observation of all modalities in scenarios with more than two modalities, because such information is not required for pairwise sufficiency, building on~\cite{wang2022rethinking}.
Several recent methods have moved beyond pairwise formulations.
CoMM~\cite{dufumier2024align} allows synergy to emerge implicitly through augmented fusion, while Symile~\cite{saporta2024contrasting} and ConFu~\cite{koutoupis2025more} target higher-order dependence via total correlation or subset fusion.
However, these approaches either lack explicit guarantees for synergy preservation or model higher-order dependence without formally connecting it to task-relevant synergistic information.

To bridge this gap, we propose Higher-order Representation and Information Learning (\model),
a synergy-aware fusion framework for multimodal contrastive representation learning.
Guided by our analysis linking synergistic information to higher-order statistical dependence, 
\model models multi-way interactions explicitly by constructing an empirical cross-moment tensor over modality embeddings.
\model then applies Tucker decomposition to obtain a compact interaction core and introduces a synergy-aware regularizer
that prevents degenerate interaction structures and preserves diverse higher-order interaction patterns during training.

\begin{wrapfigure}{r}{0.4\textwidth}
  \vspace{-1.2em}
  \centering
  \includegraphics[scale=0.59]{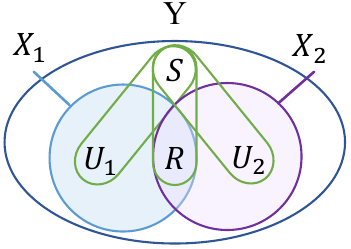}
  \caption{
  PID decomposes joint mutual information \(I(X_1,X_2;Y)\) into shared \(R\), unique \(U_1/U_2\), and synergistic \(S\).}
  \label{fig:PID}
  \vspace{-0.8em}
\end{wrapfigure}

Our main contributions are threefold.
\textbf{(1)} 
We formally establish the insufficiency of pairwise contrastive learning for synergy recovery among more than two modalities and identify the pure synergistic information as a conditional total-correlation difference, providing a precise information-theoretic criterion for when higher-order modeling is necessary.
\textbf{(2)} 
We propose \model, 
which models higher-order interactions 
via empirical cross-moment tensorization and Tucker decomposition, coupled with a synergy-aware regularizer
that penalizes excessive energy concentration in the Tucker core to preserve capacity for multi-way interactions.
\textbf{(3)} 
We conduct extensive experiments on real-world benchmarks spanning healthcare and affective computing
 --- which contain diverse modalities (\eg, image, text, audio, \textit{etc.}) --- 
complemented by controlled synthetic evaluations.
Our results demonstrate consistent improvements over state-of-the-art methods, with pronounced gains on synergy-heavy tasks.

\section{Theoretical Analysis}
\label{sec:theoretical_analysis}

In this section, 
we first formalize a latent-factor framework to 
demonstrate the insufficiency of pairwise contrastive learning for capturing synergy in scenarios involving more than two modalities,
and then quantify the synergy gap via conditional total correlation, deriving an information-theoretic criterion for the necessity of higher-order modeling.
In our theoretical analysis, we focus on the case of more than two modalities.
Our analysis of pairwise contrastive learning focuses on its minimal sufficient representations.
All proofs are provided in Appendix \ref{sec:all_proof}.

\subsection{A Latent-Factor Framework for Multimodal Representations}

\textbf{Task setup.} Given \(X=\{X_1, X_2, \dots, X_M\}\) with \(M\) modalities and a downstream task label \(Y\), our objective is to learn a multimodal representation \(\vectorname{Z}=f(X)\) that is informative for \(Y\). 
Throughout the paper, we use \(X_{1:M}:=(X_1,\dots,X_M)\) and \(X_A:=\{X_i: i\in A\}\) for any index set \(A\subseteq\{1,\dots,M\}\).

Partial Information Decomposition (PID)~\cite{williams2010nonnegative, bertschinger2014quantifying} conceptually decomposes task-relevant information into redundancy, uniqueness, and synergy (see Fig.~\ref{fig:PID} for a visual depiction). 
Motivated by this view, we adopt the latent-factor framework below.
\begin{assumption}[Tripartite Latent Factor Decomposition]
  We assume the existence of latent factors
  \(
  \{{\vectorname{q}_r, \vectorname{q}_1,\dots,\vectorname{q}_M, \vectorname{q}_{\text{syn}}}\}
  \),
  and all task-relevant information is mediated by these latent factors, \ie,
  \[
    Y \perp X_{1:M} \mid (\vectorname{q}_r,\vectorname{q}_1,\dots,\vectorname{q}_M,\vectorname{q}_{\text{syn}}),
  \]
  where \(\vectorname{q}_r\) is shared factors across modalities, 
  \(\vectorname{q}_m\) denotes the \(m\)-th modality-specific factor, 
  \ie, \(X_{-m}\perp\vectorname{q}_m\mid(\vectorname{q}_r,\vectorname{q}_{-m},
  \vectorname{q}_{\mathrm{syn}})\), where \(X_{-m}=X_{\{1,\ldots,M\}\setminus\{m\}}\). And
  \(\vectorname{q}_{\text{syn}}\) denotes synergistic factors that capture task-relevant information emerging only from joint observation.
  These latent factors are mutually independent.
  Throughout, \(\perp\) denotes conditional independence.
  \label{Ass:syn_multi_factor}
\end{assumption}

Based on the above, we further state that \((\vectorname{q}_r,\vectorname{q}_{1:M},\vectorname{q}_{\mathrm{syn}})\) represents a decomposition of the task-relevant information contained in the complete multimodal observation \(X_{1:M}\). Accordingly, once \(X_{1:M}\) is given, \((\vectorname{q}_r,\vectorname{q}_{1:M},\vectorname{q}_{\mathrm{syn}})\) provides no further information about \(Y\), which is expressed as
\[
  Y\perp (\vectorname{q}_r,\vectorname{q}_{1:M},\vectorname{q}_{\mathrm{syn}})\mid X_{1:M}.
\]
\vspace{-\baselineskip}

Among these latent factors, \(\vectorname{q}_{\text{syn}}\) is substantially harder to capture because it corresponds to task-relevant dependencies that emerge only through joint multimodal observation and therefore requires a more precise characterization~\cite{dufumier2024align,wen2025infmasking}.
We formalize this accessibility constraint as follows.
\begin{definition}[Pure Synergistic Factor]
We call \(\vectorname{q}_{\text{syn}}\) a pure synergistic factor relative to \(\vectorname{q}_r\) if
\[
I(\vectorname{q}_{\text{syn}}; X_{1:M} \mid \vectorname{q}_r) > 0, \quad \text{and} \quad I(\vectorname{q}_{\text{syn}}; X_A \mid \vectorname{q}_r) = 0, \ \forall\, A \subsetneq \{1,\dots,M\}.
\]
\label{def:pure_synergy}
\end{definition}
\vspace{-\baselineskip}
Under Assumption~\ref{Ass:syn_multi_factor}, the observations are generated as \(X_{1:M}\sim p(X_{1:M}\mid \vectorname{q}_r,\vectorname{q}_{1:M},\vectorname{q}_{\text{syn}})\), where each unique factor \(\vectorname{q}_m\) influences only \(X_m\). 
Motivated by this latent-factor framework, we interpret the learned multimodal representation \(\vectorname{Z}\) as being organized into three functional subspaces:
\begin{equation}
\vectorname{Z}\approx V^r f_r(\mathbf{q}_r) + V^u f_u(\mathbf{q}_1, \mathbf{q}_2, \dots, \mathbf{q}_M) + V^{\text{syn}} f_{\text{syn}}(\mathbf{q}_{\text{syn}}) + \epsilon,
\label{eq:factor_decomposition}
\end{equation}
where \(V^r\) spans a shared subspace, \(V^u\) captures modality-specific structures, \(V^{\text{syn}}\) denotes a synergistic subspace, and \(\epsilon\) is a noise term.

\subsection{Minimal Sufficient Representation Analysis}

Let \(\vectorname{Z}_{\text{pair}}^{\min}\) denote a representation that is \emph{sufficient} for all pairwise inter-modal alignments and \emph{minimal} in the sense of minimizing \(I(\vectorname{Z};X_{1:M})\) over the class of such representations.
Inspired by~\cite{wang2022rethinking}, we adopt the following shared information characterization for our analysis:
\begin{equation}
  I(\vectorname{Z}_{\text{pair}}^{\min};X_{1:M}\mid\vectorname{q}_r)=0,
  \label{eq:pairwise_shared_condition}
\end{equation}
meaning that \(\vectorname{Z}_{\text{pair}}^{\min}\) retains only information captured by the shared factor \(\vectorname{q}_r\).

Prior work~\cite{tian2020makes,wang2022rethinking} has shown that minimal sufficient representations in pairwise contrastive learning can be inadequate for downstream tasks.
Here, we sharpen this limitation within this latent-factor framework (Assumption~\ref{Ass:syn_multi_factor}), isolating the synergistic component 
excluded by pairwise minimal sufficiency.

\begin{restatable}[Insufficiency of Pairwise Minimal Sufficient Representation]{theorem}{pairwiseInsufficientRestated}
\label{thm:pairwise_insufficient}
Under Assumption~\ref{Ass:syn_multi_factor}, suppose that \(\vectorname{q}_{\text{syn}}\) is a pure synergistic factor in the sense of Definition~\ref{def:pure_synergy} and that \(I(Y;\vectorname{q}_{\text{syn}}\mid \vectorname{q}_r,\vectorname{q}_{1:M})>0\). Then
\[
I(Y;\vectorname{Z}_{\text{pair}}^{\min}) \le I(Y;\vectorname{q}_r).
\]
Moreover, the total information gap is lower bounded as
\[
I(Y;X_{1:M}) - I(Y;\vectorname{Z}_{\text{pair}}^{\min})
\;\ge\;
\underbrace{I(Y;\vectorname{q}_{1:M}\mid\vectorname{q}_r)}_{\text{unique gap}}
\;+\;
\underbrace{I(Y;\vectorname{q}_{\text{syn}}\mid \vectorname{q}_r,\vectorname{q}_{1:M})}_{\text{synergistic gap}}.
\]
\end{restatable}

Theorem~\ref{thm:pairwise_insufficient} reveals that pairwise contrastive learning even after recovers all shared and unique factors \((\vectorname{q}_r,\vectorname{q}_{1:M})\), an irreducible synergistic gap \(I(Y;\vectorname{q}_{\text{syn}}\mid \vectorname{q}_r,\vectorname{q}_{1:M})>0\) persists and can only be closed by information about \(\vectorname{q}_{\text{syn}}\).
Inspired by~\cite{koutoupis2025more}, we turn to total correlation~\cite{watanabe1960information} as a measure of higher-order dependence to 
characterize pure synergistic information.
Specifically, total correlation characterizes the overall statistical coupling among multiple random variables by the Kullback--Leibler divergence between their joint distribution and the product of their marginals.
The conditional difference below makes its connection to pure synergy explicit.

\begin{restatable}[Conditional Total Correlation Decomposition]{proposition}{conditionalTCRestated}
\label{prop:ctc}
For \(M\) modalities, the conditional total correlation admits the chain-rule decomposition
\[
\mathrm{TC}(X_{1:M}\mid \vectorname{q}_r)
= \sum_{k=2}^{M} I(X_k;\,X_1,\dots,X_{k-1}\mid \vectorname{q}_r).
\]
In particular, for \(M=3\): \(\mathrm{TC}(X_1,X_2,X_3\mid \vectorname{q}_r)= I(X_1;X_2\mid \vectorname{q}_r) + I(X_3;X_1,X_2\mid \vectorname{q}_r)\), and symmetrically for all permutations.
\end{restatable}
Building on this decomposition, we now establish the precise quantitative link between synergistic information and higher-order statistical dependence.
\begin{restatable}[Synergy Total Correlation Identity]{theorem}{synergyTCRestated}
\label{thm:higher_order_needed}
Under Definition~\ref{def:pure_synergy}, the synergistic information satisfies the identity
\[
I(\vectorname{q}_{\text{syn}};X_{1:M}\mid \vectorname{q}_r)
= \mathrm{TC}(X_{1:M}\mid \vectorname{q}_r,\vectorname{q}_{\text{syn}})
  - \mathrm{TC}(X_{1:M}\mid \vectorname{q}_r).
\]
\end{restatable}

\textbf{Interpretation.} We unpack Theorem~\ref{thm:higher_order_needed} in the following steps.
\textbf{(1) Synergy changes joint dependence without altering single-modality marginals.}
By Definition~\ref{def:pure_synergy}, conditioning on \(\vectorname{q}_{\text{syn}}\) leaves each marginal entropy \(H(X_m\mid\vectorname{q}_r)\) unchanged, but reduces the joint entropy \(H(X_{1:M}\mid\vectorname{q}_r)\).
Since \(\mathrm{TC}=\sum_m H(X_m\mid\cdot)-H(X_{1:M}\mid\cdot)\), this reduction increases TC exactly by \(I(\vectorname{q}_{\text{syn}};X_{1:M}\mid\vectorname{q}_r)\).
\textbf{(2) The resulting dependence is higher-order.}
Conditional on \(\vectorname{q}_{syn}\) also preserves every pairwise marginal distribution 
given \(\vectorname{q}_r\); 
hence the TC difference in Theorem~\ref{thm:higher_order_needed} can only be attributed to irreducible multi-way dependence, which pairwise objectives cannot access.

\section{Method}

\begin{figure*}[t]
  \centering
  \includegraphics[width=\textwidth, height=8cm]{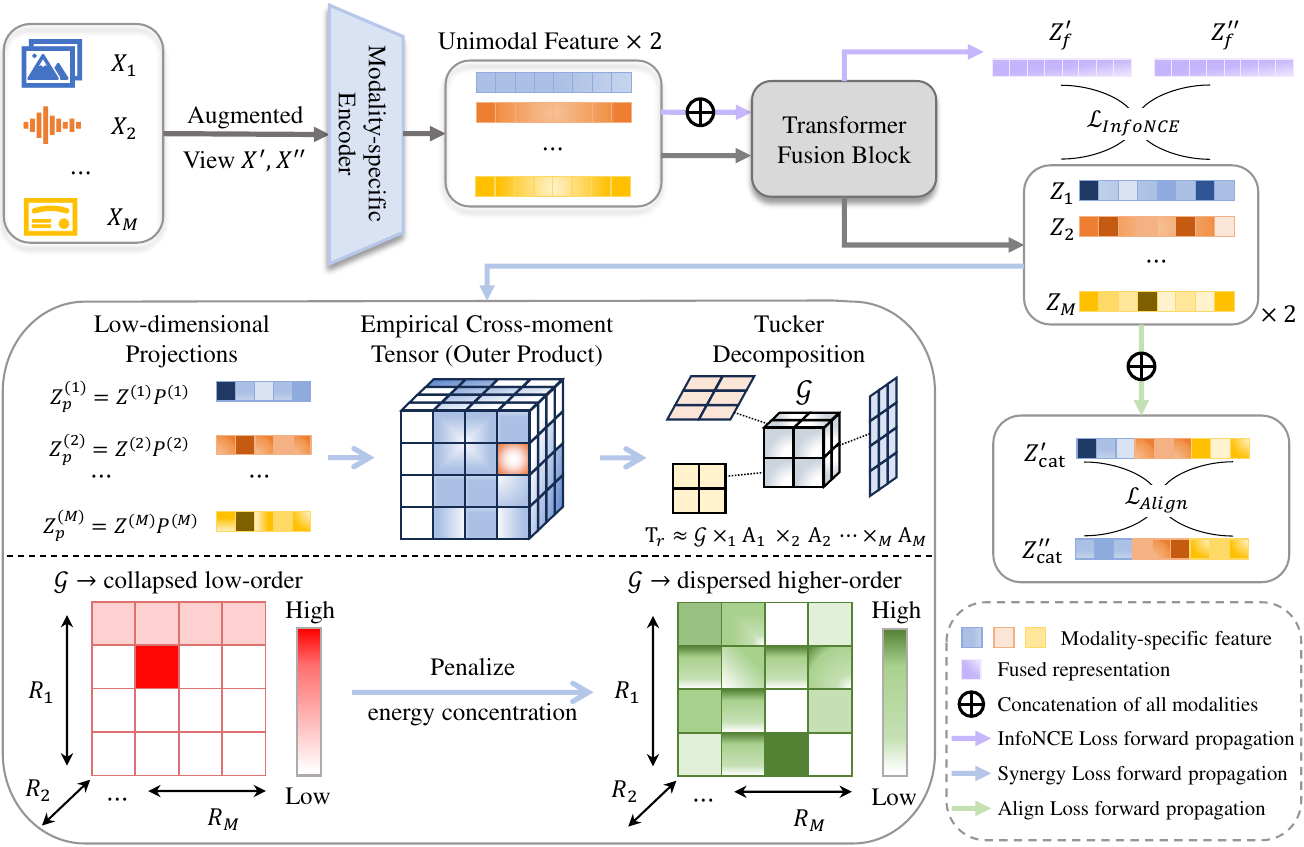}
  \caption{The overall pipeline of \model. 
  Given \(M\) modalities, we apply data augmentation to construct two views \(X^{\prime}\) and \(X^{\prime\prime}\), and encode them into unimodal representations \(\{\vectorname{Z}_m^{\prime}\}_{m=1}^{M}\) and \(\{\vectorname{Z}_m^{\prime\prime}\}_{m=1}^{M}\), as well as fused representations \(\vectorname{Z}_f^{\prime}\) and \(\vectorname{Z}_f^{\prime\prime}\). The contrastive objective \(\mathcal{L}_{\text{InfoNCE}}\) performs cross-view alignment between each unimodal representation and the fused one, while \(\mathcal{L}_{\text{syn}}\) and \(\mathcal{L}_{\text{align}}\) are computed from unimodal representations in both views.
  }
  \label{fig:flow_chart}
\end{figure*}

The overall pipeline of \model is illustrated in Fig.~\ref{fig:flow_chart}.
Let \(\vectorname{h}_v^{(m)}=f_m(X_v^{(m)})\) denote the unimodal encoder features, a shared Transformer \(g\) produces \(\vectorname{Z}_{v,u}^{(m)}=g(\vectorname{h}_v^{(m)})\) and \(\vectorname{Z}_{v,f}=g(\vectorname{h}_v^{(1)},\ldots,\vectorname{h}_v^{(M)})\), where \(v\in\{1,2\}\) indexes the augmented views.
Given the encoded modality embeddings, \model optimizes three complementary components.
First, a modality-fusion contrastive objective aligns each unimodal representation with the fused representation to extract cross-modal task-relevant signals.
Second, the synergy-aware regularization preserves the capacity of the fused representation to encode higher-order synergistic interactions.
Third, an auxiliary cross-view alignment objective stabilizes 
the shared and modality-specific subspaces while preventing representational collapse.

\subsection{Higher-Order Dependence Modeling for the Synergistic Subspace}

Theorems~\ref{thm:pairwise_insufficient} and~\ref{thm:higher_order_needed} isolate purely cross-modal, higher-order dependence as the sole information-theoretic carrier of synergy.
To translate this criterion into a practical objective, we require a mechanism to distinguish genuine multi-way interactions in the embedding space.
Inspired by~\cite{anandkumar2014tensor,kolda2009tensor, zadeh2017tensor}, we employ the \emph{empirical cross-moment tensor} that aggregates cross-modal dependencies across all interaction orders.
Here \(\otimes\) denotes the vector outer product and \(\times_m\) denotes multiplication along tensor mode \(m\).
For a fixed augmented view, given a mini-batch of \(B\) samples with modality embeddings \(\vectorname{Z}_m^{(b)}\in\mathbb{R}^{d_m}\),
the \(M\)-th order empirical cross-moment tensor \(\mathcal{T}\) is defined as
\begin{equation}
\mathcal{T} := \frac{1}{B}\sum_{b=1}^{B} \vectorname{Z}_1^{(b)} \otimes \vectorname{Z}_2^{(b)} \otimes \cdots \otimes \vectorname{Z}_M^{(b)} \;\in\; \mathbb{R}^{d_1 \times d_2 \times \cdots \times d_M},
\label{eq:cross_moment_tensor}
\end{equation}

whose \((i_1,\dots,i_M)\)-th entry estimates the mixed moment \(\mathbb{E}[Z_{1,i_1}\cdots Z_{M,i_M}]\). 
Its Tucker decomposition~\cite{kolda2009tensor} \(\mathcal{T}\approx\mathcal{G}\times_1\mathbf{A}_1\cdots\times_M\mathbf{A}_M\), 
where \(\mathbf{A}_m\in\mathbb{R}^{d_m\times r_m}\) are factor matrices, 
isolates the core tensor \(\mathcal{G}\) whose \emph{multilinear rank} characterizes the multilinear complexity of the modeled cross-moment tensor.
If the modalities interact only through a low-dimensional shared factor, \(\mathcal{T}\) inherits a low multilinear rank and its energy concentrates into a few separable components;
preserving capacity for higher-order dependencies motivates discouraging excessive energy concentration
in \(\mathcal{G}\).
This suggests a natural regularization strategy: \textbf{penalize energy concentration in \(\mathcal{G}\) to preserve representational capacity for higher-order interactions, and thus for synergy}. The following definition and lemmas formalize this diagnostic.

\begin{definition}[Redundant Low-Order Structure]
    \label{def:infor_red_def}
    Multimodal embeddings \(\{\vectorname{Z}_i\}_{i=1}^M\) are said to exhibit a \emph{redundant low-order structure} if \(\vectorname{Z}_i \approx \sum_{j \neq i} \mathbf{W}_{ij} \vectorname{Z}_j\) with \(\operatorname{rank}(\mathbf{W}_{ij}) = r_{ij} \ll d_j\), 
    where the approximation holds in the population sense: the reconstruction residual \(\boldsymbol{\delta}_i \coloneqq \vectorname{Z}_i - \sum_{j \neq i} \mathbf{W}_{ij} \vectorname{Z}_j\) satisfies \(\mathbb{E}[\|\boldsymbol{\delta}_i\|^2] \to 0\) under the shared-factor dominance regime.

\end{definition}

Under Definition~\ref{def:infor_red_def}, when the shared component dominates the embedding energy, the low-rank predictor \(\mathbf{W}_{ij}\) explains nearly all of \(\vectorname{Z}_i\), and the embeddings collapse into the redundant low-order regime.
We formally state this result in Proposition~\ref{prop:redundancy_low_rank}.
The following two lemmas 
connect redundant low-order structure to multilinear complexity and characterize single-component degeneration of the interaction core.

\begin{restatable}[Redundancy Implies Low Multilinear Complexity]{lemma}{lowRankRestated}
    \label{lemma:low_rank_red}
    If the embeddings exhibit a redundant low-order structure in the sense of Definition~\ref{def:infor_red_def}, then the cross-moment tensor \(\mathcal{T}=\mathbb{E}[\bigotimes_{m}\vectorname{Z}_m]\) admits a low multilinear-complexity approximation. In particular, the Tucker core \(\mathcal{G}\) of \(\mathcal{T}\) has at least one mode whose effective multilinear rank satisfies: 
    \[
    \operatorname{rank}_{(k)}(\mathcal{G})\leq \sum_{j \neq i} r_{ij} \quad \text{for modality } k = i.
    \]
\end{restatable}

\begin{restatable}[Energy Concentration Indicates Degenerate Structure]{lemma}{energyConcentrationRestated}
    \label{lemma:energy_concentration}
    Let
    \[
    \rho(\mathcal{G}) := \max_{i_1,\dots,i_M}\frac{\mathcal{G}_{i_1,\dots,i_M}^2}{\|\mathcal{G}\|_F^2}.
    \]
    If \(\rho(\mathcal{G}) \ge 1-\delta\) for some \(\delta\in[0,1]\), then \(\mathcal{G}\) is within relative Frobenius error at most \(\sqrt{\delta}\) of a one-sparse core tensor.
    For small \(\delta\), this indicates dominance by a single separable component in the Tucker core.

\end{restatable}

Definition~\ref{def:infor_red_def}, Lemma~\ref{lemma:low_rank_red}, and Lemma~\ref{lemma:energy_concentration} connect
information theory and representation geometry:
when shared factors dominate the embeddings, the cross-modal predictor is low-rank, which constrains the multilinear rank of the core tensor \(\mathcal{G}\); 
extreme energy concentration indicates dominance by a single separable component.
Penalizing this concentration therefore preserves the representational capacity needed to encode the higher-order dependence that, by Theorem~\ref{thm:higher_order_needed}, carries pure synergistic information.

To avoid the cost of computing the full-dimensional cross-moment tensor,
we first project each embedding into an \(r_p\)-dimensional lower-dimensional subspace.
For each augmented view \(v\in\{1,2\}\), we define \(\vectorname{Z}_{p,v}^{(m,b)}=\vectorname{Z}_{v,u}^{(m,b)}\mathbf{P}^{(m)}\), where \(\mathbf{P}^{(m)}\in\mathbb{R}^{d_m\times r_p}\) is constructed by QR decomposition of a random Gaussian matrix with independent entries and \(r_p < \min_m d_m\).
This projection preserves pairwise geometry up to a multiplicative distortion factor by the Johnson--Lindenstrauss lemma (Lemma~\ref{lemma:JL})~\cite{dasgupta2003elementary}. 
We then compute the projected empirical cross-moment tensor for each view over the mini-batch:
\begin{equation}
\mathcal{T}_{r_p,v} := \frac{1}{B}\sum_{b=1}^{B} \vectorname{Z}_{p,v}^{(1,b)} \otimes \vectorname{Z}_{p,v}^{(2,b)} \otimes \cdots \otimes \vectorname{Z}_{p,v}^{(M,b)} \;\in\; \mathbb{R}^{r_p \times r_p \times \cdots \times r_p},
  \label{eq:lower_dim_tensor}
\end{equation}
and perform Tucker decomposition \(\mathcal{T}_{r_p,v} \approx \mathcal{G}_v\times_1\mathbf{A}_{1,v}\cdots\times_M\mathbf{A}_{M,v}\), yielding the core tensor \(\mathcal{G}_v\).

To prevent excessive core-energy concentration, we introduce the following margin-based synergy-aware regularization grounded in Lemma~\ref{lemma:energy_concentration}:
\begin{equation}
\mathcal{L}_{\text{syn}} = \frac{1}{2}\sum_{v=1}^{2}\left[\xi - \left(1 - \rho(\mathcal{G}_v)\right)\right]^+,
  \label{eq:syn_loss}
\end{equation}
where \(\xi \in (0,1)\) is a margin hyperparameter controlling the minimum energy dispersion.
Throughout this paper, we adopt the notation \([\cdot]^+ = \text{ReLU}(\cdot)\).
By penalizing \(\rho(\mathcal{G}_v)>1-\xi\) in each view, \(\mathcal{L}_{\text{syn}}\) discourages one-sparse degeneration in the Tucker cores, thereby preserving multi-way coupling capacity.

While \(\mathcal{L}_{\text{syn}}\) preserves the capacity for higher-order dependencies, it does not actively recover synergistic information.
To learn such information,
we employ a contrastive alignment strategy from~\cite{dufumier2024align}, which aligns each unimodal representation with the fused multimodal representation, \ie,
\begin{equation}
  \mathcal{L}_{\text{InfoNCE}}^{(m)}
  :=
  -\frac{1}{2}\!\left(
  \hat{I}_{\text{NCE}}(\vectorname{Z}_{1,u}^{(m)},\vectorname{Z}_{2,f})
  +
  \hat{I}_{\text{NCE}}(\vectorname{Z}_{2,u}^{(m)},\vectorname{Z}_{1,f})
  \right).
  \label{eq:infonce_two_views}
\end{equation}
Here \(\hat{I}_{\text{NCE}}\) is the InfoNCE mutual-information estimator~\cite{oord2018representation}:
\[
\hat{I}_{\text{NCE}}(\vectorname{Z}, \vectorname{Z'}) = \mathbb{E}_{\substack{\vectorname{z},\vectorname{z'}_{\text{pos}}\sim p(\vectorname{Z},\vectorname{Z'})\\ \vectorname{z'}_{\text{neg}} \sim p(\vectorname{Z'})}} \left[ \log \frac{\exp \operatorname{sim}(\vectorname{z},\vectorname{z'}_{\text{pos}})}{\exp \operatorname{sim}(\vectorname{z},\vectorname{z'}_{\text{pos}}) + \sum_{\vectorname{z'}_{\text{neg}}}\exp \operatorname{sim} (\vectorname{z}, \vectorname{z'}_{\text{neg}})}\right].
\]
We use \(\operatorname{sim}(\cdot, \cdot)\) to denote cosine similarity throughout this work.
\(\vectorname{Z}_{1,u}^{(m)}\) and \(\vectorname{Z}_{2,u}^{(m)}\) denote the \(m\)-th unimodal embeddings from the two views,
and \(\vectorname{Z}_{1,f}\), \(\vectorname{Z}_{2,f}\) denote the corresponding fused representations.
Critically, although the above loss objective only captures shared and unique information in \cite{dufumier2024align},
the proposed \(\mathcal{L}_{\text{syn}}\) 
discourages excessive energy concentration 
within the core tensor to maintain the higher-order representational capacity of \(\vectorname{Z}_f\).
Consequently, the identical contrastive loss further enables the model to extract synergistic components.

\subsection{Contrastive alignment for shared and unique subspaces}

Although Eq.~\eqref{eq:syn_loss} preserves capacity for synergistic interactions, it does not by itself control the geometry of unimodal embeddings.
This is important because multimodal InfoNCE can suffer from competing alignment directions and attraction--repulsion conflicts, leading to modality gaps and unstable optimization~\cite{yin2026towards}.
Inspired by~\cite{bardes2021vicreg, zbontar2021barlow}, we introduce an auxiliary alignment regularizer that 
encourages cross-view consistency while preserving variation in modality-specific directions.

For each view \(v\in\{1,2\}\), 
let \(\hat{\vectorname{Z}}_i \in \mathbb{R}^{B \times D}\) denote the \(\ell_2\)-normalized representation that concatenates all unimodal embeddings along the feature dimension for the \(i\)-th view,
where $B$ is the batch size and $D$ is the total feature dimension. 
The alignment objective is:
\begin{equation}
\mathcal{L}_{\text{align}} = \mathcal{L}_{\text{inv}} + \mathcal{L}_{\text{var}},
\label{eq:align_loss}
\end{equation}
where \[
\mathcal{L}_{\text{inv}} = \frac{1}{B}\sum_{b=1}^{B}\bigl(1 - \operatorname{sim}(\hat{\vectorname{Z}}_1^{(b)},\hat{\vectorname{Z}}_2^{(b)})\bigr)^2
\quad \text{and} \quad
\mathcal{L}_{\text{var}} = \frac{1}{D}\sum_{d=1}^{D}\left[\gamma - \sqrt{\operatorname{Var}\!\left(\bar{\vectorname{Z}}_d\right) + \epsilon}\right]^+,
\]
with \(\bar{\vectorname{Z}}=(\hat{\vectorname{Z}}_1+\hat{\vectorname{Z}}_2)/2\), where \(\gamma\) is the variance threshold and \(\epsilon\) is a small constant.

Among these two terms, \(\mathcal{L}_{\text{inv}}\) encourages cross-view agreement of shared factors at the sample level,
and \(\mathcal{L}_{\text{var}}\) penalizes standard deviations below \(\gamma\) (0.5 by default) along each embedding dimension to discourage coordinate-wise collapse.
Because each modality embedding \(\vectorname{Z}_v^{(m)}\) is individually normalized before concatenation, modality-specific directions are preserved in \(\hat{\vectorname{Z}}_v\); \(\mathcal{L}_{\text{var}}\) then
encourages variance along \emph{all} dimensions, helping preserve modality-specific variation during alignment.

\textbf{Overall objective.}
The final loss combines all three components:
\begin{equation}
  \mathcal{L}_{\text{total}} = \sum_{m=1}^{M}\mathcal{L}_{\text{InfoNCE}}^{(m)} + \lambda_a \mathcal{L}_{\text{align}} + \lambda_s \mathcal{L}_{\text{syn}},
  \label{eq:total_loss}
\end{equation}
where \(\lambda_a\) and \(\lambda_s\) are hyperparameters.

\section{Experiments}
\label{sec:exper}

We evaluate \model on both synthetic and real-world benchmarks to validate its effectiveness in capturing synergistic interactions and enhancing downstream task performance. 
In the synthetic setting, we follow the Trifeature dataset protocol from~\cite{dufumier2024align}, which provides a controlled environment to explicitly isolate redundancy, uniqueness, and synergy. 
To assess practical utility, we conduct extensive experiments on Multibench~\cite{liang2021multibench}, a comprehensive multimodal benchmark spanning diverse domains and modality combinations (\eg, healthcare, affective computing, \textit{etc.}). 
First, we evaluate the trimodal setting, which naturally exhibits higher-order statistical dependencies essential for synergy, thereby directly testing our method's capacity to model complex multi-way interactions.
Subsequently, we evaluate generalization in bimodal configurations, which constitute the predominant configuration in practical applications. 
For evaluation, we use linear probing, \textit{i.e.}, freezing the pre-trained
feature extractor and training a linear classifier on top of the learned representations. 
The resulting downstream performance serves as a direct proxy for the discriminative quality and generalization capacity of the learned multimodal representations.
Detailed experimental settings are provided in Appendix~\ref{sec:exp_details}.

\subsection{Experiments on Synthetic Datasets with Trifeature}

\begin{wraptable}{r}{0.64\textwidth}
  \vspace{-1em}
  \centering
  \setlength{\tabcolsep}{4pt}
  \caption{Linear probing accuracy (in \%) for redundancy (shape), uniqueness (texture), and synergy (color and texture) on the Trifeature dataset. \(^\clubsuit\) denotes results from~\cite{dufumier2024align}. }
  \label{tab:trifeature_results}
  \small
  \begin{tabular}{lccc}
    \toprule
    \textit{Model} & \textit{Redundancy$\uparrow$} & \textit{Uniqueness$\uparrow$} & \textit{Synergy$\uparrow$} \\
    \midrule
    Cross$^\clubsuit$~\cite{radford2021learning} & $\textbf{100.0}$ & $11.6$ & $50.0$ \\
    Cross+Self$^\clubsuit$~\cite{yuan2021multimodal} & $99.7$ & $86.9$ & $50.0$ \\
    FactorCL$^\clubsuit$~\cite{liang2023factorized} & $99.8$ & $62.5$ & $46.5$ \\
    CoMM~\cite{dufumier2024align} & $99.9_{\pm0.06}$ & $86.8_{\pm2.99}$ & $71.4_{\pm3.47}$ \\
    \model (ours) & $99.7_{\pm0.14}$ & $\bm{92.6_{\pm3.04}}$ & $\bm{82.3_{\pm1.61}}$ \\
    \bottomrule
  \end{tabular}
  \vspace{-0.8em}
\end{wraptable}

We follow the Trifeature~\cite{hermann2020shapes} synthetic protocol of CoMM~\cite{dufumier2024align}. 
For redundancy and uniqueness, \emph{shape} is treated as shared and \emph{texture} as modality-specific.
The dataset consists of image pairs sharing identical shapes but exhibiting independent textures. 
We formulate two linear probing objectives: (1) predicting the common shape class (redundancy), and (2) identifying the texture from a single modality (uniqueness), both yielding a 10\% random-guessing baseline. 
To evaluate synergy, we impose a deterministic cross-modal mapping $\mathcal{M}$ that links the texture of $X_1$ to the color of $X_2$ (\eg, stripes $\mapsto$ red), restricting training exclusively to pairs that adhere to this constraint.
Evaluation proceeds via a binary compliance task defined as $Y=\mathds{1}[\,(\operatorname{tex}(X_1),\operatorname{col}(X_2))\in\mathcal{M}\,]$, with a 50\% chance baseline.

Experimental results are summarized in Tab.~\ref{tab:trifeature_results}.
Cross-modality constraints with InfoNCE~\cite{radford2021learning} ("Cross") effectively capture redundant shape information but yield near-chance performance on uniqueness.
Self-supervised constraints on each encoder ("Cross+Self"), FactorCL~\cite{liang2023factorized}, and CoMM~\cite{dufumier2024align} all improve uniqueness performance, with CoMM achieving the best result as it implicitly models the synergistic factor and is the only baseline that significantly improves synergy capture.
In comparison, \model achieves the highest uniqueness and synergy performance, outperforming CoMM by 5.8\% on uniqueness and 10.9\% on synergy.

\subsection{Experiments on Multibench}

To further validate the practical utility of \model, we conduct extensive experiments on Multibench~\cite{liang2021multibench}, which includes multiple datasets with different modalities and domains.
These datasets provide a comprehensive evaluation of the model's ability to learn multimodal representations that are effective for downstream tasks.
Further dataset details are provided in Appendix~\ref{sec:dataset_details}.

\subsubsection{Experiments with 3 Modalities: UR-FUNNY, MOSI, and MOSEI}

\begin{wraptable}{r}{0.68\textwidth}
  \vspace{-1em}
  \centering
  \caption{Linear probing top-1 accuracy (in \%) for classification 
  tasks on trimodal Multibench. }
  \label{tab:tri_results}
  \small
  \begin{tabular}{lcccc}
      \toprule
      \textit{Model} & \textit{UR-FUNNY$\uparrow$} & \textit{MOSI$\uparrow$} & \textit{MOSEI$\uparrow$} & \textbf{Average} \\
      \midrule
      Tri\_CLIP & $60.6_{\pm0.79}$ & $60.1_{\pm3.43}$ & $64.1_{\pm0.42}$ & $61.6$ \\
      CMC~\cite{tian2020cmc} & $60.1_{\pm0.78}$ & $62.4_{\pm1.67}$ & $62.7_{\pm0.72}$ & $61.73$ \\
      ConFu~\cite{koutoupis2025more} & $61.3_{\pm0.96}$ & $62.5_{\pm2.66}$ & $64.8_{\pm0.61}$ & $62.87$ \\
      CoMM~\cite{dufumier2024align} & $64.8_{\pm1.13}$ & $66.0_{\pm2.27}$ & $69.9_{\pm0.34}$ & $66.9$ \\
      \model (ours) & $\bm{65.5_{\pm0.56}}$ & $\bm{67.2_{\pm1.28}}$ & $\bm{70.5_{\pm0.20}}$ & $\bm{67.73}$ \\
      \bottomrule
  \end{tabular}
  \vspace{-0.9em}
\end{wraptable}

We first evaluate the trimodal setting since it naturally exhibits higher-order statistical dependencies essential for synergy.
The data preprocessing steps can be found in Appendix~\ref{sec:dataset_details}.
Those experiments are conducted in trimodal configurations using identical encoders and modality arrangements, with all models trained directly on encoded representations.
We compare against several established multimodal contrastive learning frameworks: Tri\_CLIP (trained with pairwise contrastive objectives across three modalities), ConFu~\cite{koutoupis2025more}, CMC~\cite{tian2020cmc}, and CoMM~\cite{dufumier2024align}.
As shown in Tab.~\ref{tab:tri_results}, 
while ConFu extends Tri\_CLIP by incorporating a higher-order contrastive loss, it yields only limited performance gains. 
In contrast, \model consistently surpasses all baselines across the three datasets and outperforms the strongest competitor --- CoMM --- by 0.7\% on UR-FUNNY, 1.2\% on MOSI, and 0.6\% on MOSEI.
These empirical results demonstrate that \model's explicit modeling of higher-order interactions effectively enhances representation quality and downstream task performance, thereby validating the theoretical principles established in our information-theoretic analysis.

\subsubsection{Experiments with 2 Modalities}

Beyond the trimodal setting, 
we further evaluate whether \model remains effective in the more common bimodal scenario and test the practical robustness of our synergy-aware objective across heterogeneous domains, including healthcare and affective understanding.
Adopting the data preprocessing steps established in~\cite{liang2023factorized},
\begin{table}[ht]
  \centering
  \vspace{-1.2em}
  \caption{Linear probing top-1 accuracy (in \%) for classification tasks on bimodal Multibench. \(^\clubsuit\) denotes results are from~\cite{liang2023factorized}. }
  \label{tab:Multibench_results}
  \small
  \begin{tabular}{lccccc}
      \toprule
  \multirow{2}{*}{\textit{Model}} & \multicolumn{5}{c}{\textit{Classification}} \\
  \cmidrule(lr){2-6}
  & \textit{MIMIC}$\uparrow$ & \textit{MOSI}$\uparrow$ & \textit{UR-FUNNY}$\uparrow$ & \textit{MUStARD}$\uparrow$ & \textbf{Average}$\uparrow$\\
      \midrule
  $\text{Cross}^\clubsuit$~\cite{radford2021learning} & $66.7_{\pm0.1}$ & $47.8_{\pm1.8}$ & $50.1_{\pm1.9}$ & $53.5_{\pm2.9}$ & $54.52$ \\
  $\text{Cross+Self}^\clubsuit$~\cite{yuan2021multimodal} & $65.49_{\pm0.0}$ & $49.0_{\pm1.1}$ & $59.9_{\pm0.9}$ & $53.9_{\pm4.0}$ & $57.07$ \\
  $\text{FactorCL}^\clubsuit$~\cite{liang2023factorized} & $67.3_{\pm0.0}$ & $51.2_{\pm1.6}$ & $60.5_{\pm0.8}$ & $55.80_{\pm0.9}$ & $58.7$ \\
  CoMM~\cite{dufumier2024align} & $66.4_{\pm0.4}$ & $63.7_{\pm2.5}$ & $63.3_{\pm0.5}$ & $64.4_{\pm1.1}$ & $64.45$ \\
  \model (ours) & $\bm{68.0_{\pm0.7}}$ & $\bm{65.8_{\pm1.2}}$ & $\bm{63.9_{\pm0.9}}$ & $\bm{66.6_{\pm2.7}}$ & $\bm{66.08}$ \\
  \bottomrule
  \end{tabular}
  \vspace{-0.6em}
\end{table}
we compare our method with Cross~\cite{radford2021learning}, Cross+Self~\cite{yuan2021multimodal}, FactorCL~\cite{liang2023factorized}, and CoMM~\cite{dufumier2024align}.
Tab.~\ref{tab:Multibench_results} summarizes the results.
\model achieves the best performance on all four bimodal classification tasks.
Compared with the strongest baseline, CoMM, \model improves accuracy by 1.6 percentage points on MIMIC, 2.1 on MOSI, 0.6 on UR-FUNNY, and 2.2 on MUStARD, leading to an average gain of 1.63 percentage points.
These results indicate that the proposed synergy-aware regularization is not limited to interactions among three or more modalities; it also improves representation learning in practical bimodal settings by preserving non-degenerate cross-modal dependence.

\section{Ablation Studies}

\textbf{Loss functions.} 
To evaluate the contribution of each loss term in \model, we compare different loss combinations on the Trifeature dataset for capturing multimodal interactions.
\begin{table}[ht]
  \centering
  \vspace{-1.2em}
  \setlength{\tabcolsep}{5pt}
  \caption{
    Linear probing accuracy (in \%) of multimodal interactions on the Trifeature dataset under different combinations of loss terms in Eq.~\ref{eq:total_loss}.
  }
  \label{tab:loss_component_ablation}
  \small
  \begin{tabular}{cccccc}
    \toprule
    \multicolumn{3}{c}{\textit{Loss components}} & \multirow{2}{*}{\textit{Redundancy}\(\uparrow\)} & \multirow{2}{*}{\textit{Uniqueness}\(\uparrow\)} & \multirow{2}{*}{\textit{Synergy}\(\uparrow\)} \\
    \cmidrule(lr){1-3}
    \(\mathcal{L}_{\mathrm{InfoNCE}}\) & \(\mathcal{L}_{\mathrm{align}}\) & \(\mathcal{L}_{\mathrm{syn}}\) & & & \\
    \midrule
    \(\times\) & \(\checkmark\) & \(\checkmark\) & \(28.5_{\pm2.74}\) & \(12.5_{\pm3.15}\) & \(50.0_{\pm0.00}\) \\
    \(\checkmark\) & \(\times\) & \(\times\) & \(99.9_{\pm0.10}\) & \(87.2_{\pm2.13}\) & \(50.0_{\pm0.00}\) \\
    \(\checkmark\) & \(\times\) & \(\checkmark\) & \(99.9_{\pm0.05}\) & \(86.6_{\pm2.79}\) & \(72.4_{\pm3.16}\) \\
    \(\checkmark\) & \(\checkmark\) & \(\times\) & \(99.9_{\pm0.09}\) & \(85.1_{\pm6.86}\) & \(78.0_{\pm0.44}\) \\
    \(\checkmark\) & \(\checkmark\) & \(\checkmark\) & \(99.7_{\pm0.14}\) & \(\mathbf{92.6_{\pm3.04}}\) & \(\mathbf{82.3_{\pm1.61}}\) \\
    \bottomrule
  \end{tabular}
  \vspace{-0.6em}
\end{table}

  As shown in Tab.~\ref{tab:loss_component_ablation}, both \(\mathcal{L}_{\mathrm{InfoNCE}}\) alone and \(\mathcal{L}_{\mathrm{align}}+\mathcal{L}_{\mathrm{syn}}\) perform poorly on synergy, achieving only chance-level accuracy (50.0\%). The variant without \(\mathcal{L}_{\mathrm{InfoNCE}}\) performs worse overall, with substantially lower redundancy and uniqueness accuracy.
Adding \(\mathcal{L}_{\mathrm{syn}}\) to \(\mathcal{L}_{\mathrm{InfoNCE}}\) increases synergy accuracy to 72.4\%, while adding \(\mathcal{L}_{\mathrm{align}}\) increases it to 78.0\%. The full objective further improves synergy accuracy from 78.0\% to 82.3\% and uniqueness accuracy from 85.1\% to 92.6\% compared with \(\mathcal{L}_{\mathrm{InfoNCE}}+\mathcal{L}_{\mathrm{align}}\).
Meanwhile, redundancy remains nearly saturated. 
These results suggest that \model benefits from the complementary contributions of all components.

\begin{wraptable}{r}{0.58\textwidth}
    \vspace{-0.8em}
    \centering
    \setlength{\tabcolsep}{3.5pt}
    \caption{
      Quantitative evaluation of InterSHAP on Multibench~\cite{liang2021multibench} datasets.
      \(\times\) denotes the model without \(\mathcal{L}_{\text{syn}}\), and \(\checkmark\) denotes the full objective of \model.
    }
    \label{tab:inter_shap}
    \small
    \begin{tabular}{@{} *{5}{c} @{}}
      \toprule
      \multirow{2}{*}{\(\mathcal{L}_{\text{syn}}\)} 
      & \multicolumn{3}{c}{\textit{InterSHAP Value (\%)}} 
      & \multirow{2}{*}{\textbf{Average}} \\
      \cmidrule(lr){2-4}
      & \textit{UR-FUNNY} & \textit{MOSI} & \textit{MOSEI}  \\
      \midrule
      $\times$ & $7.78_{\pm4.56}$ & $6.82_{\pm3.73}$ & $1.35_{\pm0.53}$ & 5.32 \\
      \checkmark & $9.25_{\pm5.56}$ & $9.90_{\pm4.81}$ & $1.47_{\pm0.40}$ & 6.87 \\
      \midrule
      \textbf{Improvement} & $\bm{+1.47}$ & $\bm{+3.08}$ & $\bm{+0.12}$ & $\bm{+1.55}$ \\
      \bottomrule
    \end{tabular}
    \vspace{-0.8em}
  \end{wraptable}

  \textbf{Interaction attribution.}
  To evaluate whether the proposed synergy-aware regularization (Eq.~\eqref{eq:syn_loss}) increases cross-modal interaction attribution, we report InterSHAP~\cite{wenderoth2025measuring} as an attribution-based interaction metric.
Tab.~\ref{tab:inter_shap} reports InterSHAP values on three trimodal datasets from Multibench~\cite{liang2021multibench}.
  Without \(\mathcal{L}_{\text{syn}}\) (Eq.~\eqref{eq:syn_loss}), the interaction attribution remains modest.
  In contrast, adding the proposed synergy-aware regularization consistently raises mean InterSHAP on all datasets, improving the average score from 5.32\% to 6.87\%.
  The gains are most pronounced on MOSI and UR-FUNNY, indicating stronger use of cross-modal interaction patterns after applying the synergy-aware regularization.
  The slight increase on MOSEI indicates a smaller mean change in interaction attribution.
  Moreover, since interaction attributions in real-world benchmarks are often numerically small~\cite{wenderoth2025measuring}, the consistent upward shift across datasets provides complementary behavioral evidence that \(\mathcal{L}_{\text{syn}}\) encourages the model to exploit cross-modal interactions.
  The metric definition and further comparisons with CoMM and a multilinear-rank control are provided in Appendix~\ref{sec:intership}.

\begin{wraptable}{r}{0.62\textwidth}
  \vspace{-1.1em}
  \centering
  \setlength{\tabcolsep}{4pt}
  \caption{Linear probing accuracy (in \%) of multimodal interactions on the Trifeature dataset.
  Concat+Linear: concatenate the modality embeddings and pass through a linear layer.
  }
  \label{tab:fusion_compare}
  \small
  \begin{tabular}{lcccc}
    \toprule
    \textit{Fusion} & $R$ & $U$ & $S$ & \textbf{Average} \\
    \midrule
    Concat+Linear & $99.9_{\pm0.1}$ & $65.4_{\pm10.73}$ & $50.0_{\pm0.00}$ & $71.77$ \\
    \model & $99.7_{\pm0.14}$ & $\bm{92.6_{\pm3.04}}$ & $\bm{82.3_{\pm1.61}}$ & $\bm{91.83}$ \\
    \bottomrule
  \end{tabular}
  \vspace{-0.8em}
\end{wraptable}

\textbf{Fusion module.} 
We compare the attention-based fusion module with a linear fusion baseline on the Trifeature dataset.
As shown in Tab.~\ref{tab:fusion_compare}, 
the linear fusion module fails to capture synergistic information, whereas \model substantially improves the performance of synergy.
This result suggests that synergy-aware regularization is effective only when the fusion module has sufficient capacity to model higher-order cross-modal interactions.
It also provides empirical support for Definition~\ref{def:infor_red_def}: 
if the information in one modality can be linearly approximated from another modality, the resulting representation is unlikely to encode genuinely synergistic information.

\section{Related Work}

\textbf{Multimodal contrastive learning.}
Multimodal contrastive learning aligns cross-modal representations by maximizing mutual information through constructed positive and negative pairs~\cite{radford2021learning,oord2018representation,tian2020cmc}. 
Positive pairs are typically constructed by augmenting the same instance, leveraging intra- and cross-modal variations as mutual supervision signals~\cite{jing2020self,caron2020unsupervised,chen2020simple,zbontar2021barlow}. 
The effective utilization of the three forms of multimodal information (\ie, redundancy, uniqueness, and synergy) is driving the advancement of representation learning~\cite{wen2026dependency}.
Predominant pairwise contrastive frameworks primarily capture shared information across modalities~\cite{radford2021learning,jia2021scaling}, often at the expense of modality-unique features. 
MV-DHEL~\cite{koromilas2025principled} jointly aligns all views of the same sample while imposing view-specific uniformity to decouple alignment from uniformity.
While recent efforts have pivoted toward disentangling and preserving modality-specific information~\cite{liang2023factorized, wang2024decoupling,guo2025disentangling,qian2025decalign}, 
and methods such as CoMM~\cite{dufumier2024align} and InfMasking~\cite{wen2025infmasking} implicitly capture synergistic interactions, the explicit modeling and recovery of synergy remain an open challenge.
Building on CoMM's fusion and alignment paradigm, \model adds an explicit higher-order structural regularizer.

\textbf{Higher-order multimodal dependence.}
Recent self-supervised methods move beyond pairwise alignment by explicitly modeling multi-way cross-modal dependence.
Symile~\cite{saporta2024contrasting} uses a scalar multilinear critic to optimize a contrastive lower bound on total correlation, while ConFu~\cite{koutoupis2025more} aligns fused modality subsets with the remaining modality alongside pairwise alignment.
M3G~\cite{pmlr-v235-piran24a} promotes joint consistency across views by minimizing the gap between ground-truth and optimal matching costs under entropy-regularized multi-marginal optimal transport.
Geometric criteria provide an alternative route to joint alignment: GRAM~\cite{DBLP:conf/iclr/CicchettiGSC25} aligns modalities by minimizing the Gramian volume spanned by their embeddings, and TRIANGLE~\cite{cicchetti2025triangle} enforces tri-modal alignment by replacing cosine similarity with a triangle-area similarity computed in the embedding span.
Nevertheless, these approaches largely target joint dependence or alignment, while \model additionally regularizes the empirical cross-moment structure to preserve capacity for synergistic interactions.

\textbf{Tensor-based multimodal fusion.}
Tensor fusion provides an explicit mechanism for modeling multiplicative cross-modal interactions~\cite{anandkumar2014tensor,kolda2009tensor}.
TFN~\cite{zadeh2017tensor} enumerates unimodal, bimodal, and trimodal interactions by forming outer products of modality embeddings, but incurs exponential complexity. 
LMF~\cite{liu2018efficient} alleviates this via low-rank factorization while retaining multiplicative coupling. 
Diverging from tensor-based fusion paradigms, we instead employ Tucker decomposition~\cite{kolda2009tensor} as a \emph{regularization} prior, explicitly enforcing a non-degenerate multi-way structure in the fused representation.

\section{Conclusion}
\label{sec:conclusion}

We presented \model, a multimodal contrastive representation learning framework that addresses the synergistic gap left by pairwise contrastive learning.
By proving under the condition of pure synergy that synergistic information equals a conditional total-correlation difference recoverable only through higher-order dependence,
we motivated synergy-aware regularization on the Tucker-decomposed cross-moment core that preserves multi-way interaction capacity alongside multimodal contrastive objectives.
Experiments on synthetic and real-world benchmarks show consistent gains, especially on the controlled synergy task. 
Current limitations include the assumption of cleanly separated latent factors;
future work will explore relaxing this assumption and extending the framework to large multimodal models. 
Further discussions can be found in Appendix~\ref{sec:limitations_future_work}.

{\large\textbf{Acknowledgments}}
This work was supported by the Sichuan Provincial Innovation Group Project (Grant No. 2024NSFTD0054), the Major Science and Technology Special Project of the Sichuan Provincial Department of Science and Technology (Grant No. 2024ZDZX0002), Fundamental Research Funds for the Central Universities (JBK202511081), the National Natural Science Foundation of China (No. U24A20323 , No. 62502397, No.12526646, No 126011059). 
We thank anonymous reviewers for their insightful feedback.

\newpage
{
\small
\bibliographystyle{plain}
\bibliography{reference.bib}
}

\newpage
\appendix

\begin{center}
    \LARGE \textbf{Appendix}
\end{center}

\section{Limitations and Future Work}
\label{sec:limitations_future_work}

\textbf{Limitations.}
Our theoretical analysis relies on Assumption~\ref{Ass:syn_multi_factor}.
In particular, the core results are stated under the pure-synergy condition in Definition~\ref{def:pure_synergy}, where the synergistic factor carries no mutual information with any proper subset of modalities once the shared factor is conditioned on.
This assumption is useful for isolating the failure mode of pairwise contrastive learning and for deriving a clean connection between synergistic information and conditional total correlation. 
However, real multimodal data may exhibit mixed interaction structures: synergistic signals can coexist with residual pairwise dependence, partially redundant cues, or modality-specific shortcuts.
In such cases, our theory can be interpreted as characterizing an idealized regime that clarifies when higher-order modeling is necessary.
Although the empirical results in Sec.~\ref{sec:exper} and the InterSHAP analysis in Appendix~\ref{sec:intership} suggest that the proposed regularization is beneficial beyond purely synthetic settings, a more systematic robustness study under approximate or partially overlapping synergy remains open.

\textbf{Future Work.}
Future work can extend Definition~\ref{def:pure_synergy} from pure synergy to approximate and mixed-synergy regimes. 
One promising direction is to develop data-dependent diagnostics for estimating the degree of redundancy, uniqueness, and synergy in real benchmarks, which could guide when higher-order regularization is most useful and how strongly it should be applied.
Moreover, scaling the proposed objective to larger modality sets and integrating it with 
multimodal alignment objectives or large multimodal models~\cite{ICLR2026_3d14c058} may further test whether synergy-aware learning improves general-purpose multimodal reasoning.

\section{Experimental Details}
\label{sec:exp_details}

\textbf{Implementation details.} 
Unless otherwise specified, all methods are implemented in PyTorch and trained on a single NVIDIA RTX 4090 GPU with 24 GB of VRAM. 
We employ the AdamW optimizer~\cite{loshchilov2017decoupled} across all experiments. 
Each experimental configuration is trained over five independent runs using random seeds in the range \([42, 46]\).
We report the mean performance along with the standard deviation.
Models are trained for up to 100 epochs; the best checkpoint is selected based on validation accuracy, with early stopping applied when applicable to prevent overfitting. 
Detailed hyperparameter settings for each dataset and method are released alongside our source code.

\textbf{Differentiability and Gradient Path of Tucker Decomposition.}
\model supports end-to-end backpropagation through the Tucker decomposition by leveraging the PyTorch autograd engine coupled with the TensorLy library's differentiable backend. 
We compute the approximate Tucker decomposition \(\mathcal{T}_{r_p,v} \approx \mathcal{G}_v \times_1 A_{1,v} \cdots \times_M A_{M,v}\) using the Higher-Order Orthogonal Iteration (HOOI) algorithm~\cite{kroonenberg2008applied}, initialized via truncated SVD.
The autodiff engine dynamically traces the alternating least-squares updates and embedded SVD operations executed during each forward pass.
The solver employs a convergence tolerance (\(10^{-3}\)) on changes in reconstruction 
error alongside a maximum of 15 iterations,
ensuring that the computational graph records the executed optimization steps.
Consequently, gradients propagate backward through the actual HOOI trajectory of the executed finite iterations without truncation or manual intervention.

During backpropagation, the gradient with respect to the modality embeddings \(\vectorname{Z}_{v,u}^{(m,b)}\) is computed via the chain rule as:
\[
\frac{\partial\ell_{\mathrm{syn},v}}{\partial Z_{v,u,a}^{(m,b)}}
=\frac{1}{B}\sum_{j_1,\ldots,j_M}
\mathcal{D}_{j_1\ldots j_M}P^{(m)}_{a j_m}
\prod_{k\ne m} Z_{p,v,j_k}^{(k,b)},
\]
where the total derivative \(\mathcal{D}=\partial\ell_{\mathrm{syn},v}/\partial\mathcal{T}_{r_p,v}\) encompasses the full derivative chain through the traced HOOI steps and SVD differentials, \(\ell_{\mathrm{syn},v}\) denotes the per-view hinge in Eq.~\eqref{eq:syn_loss}, and \(\vectorname{Z}_{p,v}^{(m,b)}=\vectorname{Z}_{v,u}^{(m,b)}P^{(m)}\).
This term is evaluated automatically by the framework via unrolled automatic differentiation, requiring no manual Jacobian derivation.
Autograd also accounts for the factor \(1/2\) from two-view averaging.
The projection matrix \(P^{(m)}\) is constructed once via QR decomposition of a random Gaussian matrix with a fixed seed and cached for both views and remains constant throughout training.
As a static linear transformation, it introduces no additional optimization variables and provides a gradient path from the interaction core back to the trainable representation modules.

We provide the analysis of computational and optimization efficiency in Appendix~\ref{sec:computational_complexity}.
Notably, the synergy-aware regularization \(\mathcal{L}_{\text{syn}}\) employs a margin-based activation that, when numerical clamps are inactive, remains dormant in each view unless the core tensor energy concentration exceeds the prescribed threshold \(1-\xi\).
The implementation floors \(\|\mathcal{G}_v\|_F^2\) at \(10^{-6}\), clips the remaining-energy ratio to \([10^{-6},1-10^{-6}]\), and averages the two per-view hinges.
This sparse activation mechanism sets the direct hinge gradient to zero when inactive, while the implementation still evaluates forward and backward operations through the Tucker decomposition module.

\textbf{Model configuration.}
For dataset-specific encoder architectures, modality-specific data augmentations, and latent converters, we follow the same configurations as CoMM~\cite{dufumier2024align}. 

The fusion module in \model adopts a Transformer-based encoder to integrate modality-specific representations, following the design of CoMM~\cite{dufumier2024align}.
Each layer comprises a multi-head self-attention mechanism followed by a position-wise feed-forward network, both equipped with residual connections and layer normalization. 
In bimodal settings, we employ a single-layer Transformer with 8 attention heads, whereas trimodal configurations utilize a 2-layer Transformer with the same head count.
A learnable \texttt{[CLS]} token is prepended to the input sequence to aggregate cross-modal information into a unified global representation. 
For ConFu~\cite{koutoupis2025more}, we replicate the original implementation by using a 2-layer MLP as the subset fusion module.

\section{Dataset Details}
\label{sec:dataset_details}

\subsection{Trifeature}

The Trifeature dataset~\cite{hermann2020shapes} serves as a controlled benchmark for probing feature learning in visual neural networks. 
It comprises three distinct attributes---shape, color, and texture---each instantiated with 10 categorical variants, yielding 1,000 unique compositions. 
We adopt the standard split of 800 compositions for training and 200 for evaluation.
To enhance sample diversity, each training composition is augmented with three randomized instances, where shapes undergo independent rotations uniformly sampled from \([-45^\circ, 45^\circ]\). All shapes are rendered within a \(128 \times 128\) region, then randomly positioned on a \(224 \times 224\) canvas while preserving full visibility; 
color and texture are applied analogously. 
Following~\cite{dufumier2024align}, we construct bimodal pairs from these instances, resulting in 10,000 training and 4,096 test pairs drawn \textit{i.i.d.} from the same generative distribution.

\subsection{Multibench}

According to~\cite{liang2021multibench}, all datasets below have been systematically de-identified to eliminate personally identifiable information and ensure strict privacy compliance. 
Multibench~\cite{liang2021multibench} provides access to the extracted features and labels for each dataset. 

\begin{itemize}
    \item \textbf{MIMIC}~\cite{johnson2016mimic} is a large-scale critical care dataset comprising de-identified clinical records from over 38,000 patients admitted to intensive care units between 2001 and 2012. 
    Each sample integrates two heterogeneous modalities: 
    (i) a temporal modality consisting of 24-hour longitudinal clinical measurements (12-dimensional vectors), 
    and (ii) a static tabular modality encoding demographic features (5-dimensional vectors). 

    We follow the protocol of prior work~\cite{liang2021multibench,liang2023factorized,dufumier2024align} and formulate the task as binary disease classification: predicting whether a patient's diagnosis falls within ICD-9 (\textit{International Statistical Classification of Diseases and Related Health Problems}) code group 7 (codes 460--519). 

    \item \textbf{MOSI}~\cite{zadeh2016mosi} contains 2,199 YouTube-sourced video clips curated for multimodal sentiment analysis. Each sample comprises three modalities: visual, audio, and text.
    Ground-truth sentiment scores are originally annotated on a continuous scale from -3 to 3; 
    following~\cite{liang2023factorized}, these labels are binarized into positive and negative classes. 

    In our experiments, we utilize the textual and visual modalities in bimodal configurations, while the trimodal setting incorporates all three modalities to evaluate the model's ability to integrate complementary information across modalities.

    \item \textbf{MOSEI}~\cite{zadeh2018mosei} is a large-scale multimodal benchmark for sentence-level sentiment analysis and emotion recognition, containing almost 23,000 monologue video clips. 
    It comprises more than 65 hours of annotated video content, spanning 250 distinct categories and featuring contributions from over 1,000 speakers. 
    Each sample includes aligned visual, acoustic, and textual modalities, with sentiment labels originally annotated on a continuous scale from -3 to 3.

    In our experiments, we use its sentiment annotations and, following~\cite{liang2023factorized}, binarize them into positive and negative classes. We then evaluate the model's performance in trimodal settings.

    \item \textbf{UR-FUNNY}~\cite{hasan2019ur} is a large-scale benchmark for humor detection, curated from 1,866 TED talks and containing 16,514 multimodal samples. 
    Each instance integrates visual frames, acoustic signals, and aligned textual transcripts, with the goal of discerning humorous versus non-humorous content. 

    In our experiments, we utilize the textual and visual modalities in bimodal configurations, while the trimodal setting incorporates all three modalities.

    \item \textbf{MUStARD}~\cite{castro2019mustard} is a multimodal benchmark for sarcasm detection, compiled from clips of popular sitcoms such as \textit{Friends}. 
    The corpus contains 690 samples, each providing synchronized video, audio, and subtitle streams with binary sarcasm annotations. 

    In our experiments, we utilize the textual and visual modalities in bimodal configurations.

    \item 
      \textbf{Vision\&Touch}~\cite{lee2020making} is a robotic manipulation benchmark with 150 trajectories of 1,000 time steps each. Its observations combine RGB and depth imagery with force measurements and end-effector position and velocity.

    In our experiments, we utilize the visual and proprioceptive modalities in a bimodal binary classification task to predict whether contact will occur at the next time step.

\end{itemize}

\subsection{\texorpdfstring{MM-IMDb}{MM-IMDb}}

Multimodal IMDb (MM-IMDb)~\cite{arevalo2017gated} poses movie-genre prediction as a multilabel classification task with 23 categories. Its 25,959 movies are drawn from MovieLens 20M~\cite{movielensdataset} and include posters, plot descriptions, genre annotations, and metadata. We use the posters and plot descriptions as the image and text modalities, respectively, and evaluate weighted and macro F1. Although the dataset is included in MultiBench~\cite{liang2021multibench}, our experiments train on its original image and text inputs.

\section{Additional Experiments}
\label{sec:additional_experiments}

\subsection{Evaluations with InterSHAP}
\label{sec:intership}

  \textbf{Metric definition.}
Rather than scoring only downstream task performance, this measure inspects how predictions are formed by separating cross-modal interaction effects from modality-specific main effects.
Concretely, predictions \(f(x)\) are decomposed with the Shapley Interaction Index (SII), where pairwise terms represent interaction strength and diagonal terms represent unimodal contributions.
Let \(M\) be the number of modalities and let \(\phi_{ij}\) denote the SII attribution between modalities \(i\) and \(j\) (computed either per sample or aggregated over the dataset).
The InterSHAP score~\cite{wenderoth2025measuring} is then given by:
\begin{equation}
    \text{InterSHAP} = \frac{\sum_{\substack{i,j=1 \\ i \neq j}}^{M} |\phi_{ij}|}{\sum_{i,j=1}^{M} |\phi_{ij}|},
\end{equation}
where the numerator aggregates off-diagonal interaction attributions (\(i\neq j\)) and the denominator aggregates all absolute attributions, including diagonal main effects (\(i=j\)).
Larger values indicate that model decisions rely more on interaction terms than on unimodal effects.

Table~\ref{tab:rank_intershap} extends the comparison in Table~\ref{tab:inter_shap} with CoMM and a multilinear-rank control on the UR-FUNNY and MOSI datasets,
where the multilinear-rank control replaces \(\mathcal{L}_{\mathrm{syn}}\) with a low-rank penalty on the Tucker core while keeping the other losses unchanged.
We use this control to test whether a multilinear-rank restriction alone explains the mean InterSHAP gains. Its loss is
\(\mathcal{L}_{\mathrm{rank}}=\sum_{m=1}^{M}\sum_{i>k_m}\sigma_i^2(\mathcal{G}_{(m)})\). 
The penalty sums the squared singular values beyond the leading \(k_m\) singular values of each mode unfolding \(\mathcal{G}_{(m)}\).
This control encourages low multilinear rank, whereas \(\mathcal{L}_{\mathrm{syn}}\) activates only when one core coordinate exceeds the prescribed energy share.

\begin{table}[ht]
  \centering
  \caption{Further comparison with InterSHAP on the UR-FUNNY and MOSI datasets.}
  \label{tab:rank_intershap}
  \small
  \begin{tabular}{lcc}
    \toprule
    \textit{Model} & \textit{UR-FUNNY} & \textit{MOSI} \\
    \midrule
    CoMM & \(7.77_{\pm3.05}\) & \(3.82_{\pm1.90}\) \\
    HRIL without \(\mathcal{L}_{\mathrm{syn}}\) & \(7.78_{\pm4.56}\) & \(6.82_{\pm3.73}\) \\
    Multilinear-rank control & \(8.17_{\pm6.33}\) & \(5.87_{\pm4.26}\) \\
    HRIL & \(9.25_{\pm5.56}\) & \(9.90_{\pm4.81}\) \\
    \bottomrule
  \end{tabular}
\end{table}

  As shown in Table~\ref{tab:rank_intershap}, HRIL has the highest mean InterSHAP score on both datasets. 
  The multilinear-rank control penalty yields a small mean increase over the variant without synergy regularization on UR-FUNNY but a decrease on MOSI. 
  \model exceeds it by 1.08 and 4.03 percentage points, respectively.
  These results support preserving higher-order interaction capacity within the prescribed Tucker subspace. The regularizer \(\mathcal{L}_{\mathrm{syn}}\) discourages excessive dominance by a single core coordinate while allowing the core to remain sparse.

\subsection{Hyperparameter Sensitivity}
\label{sec:hyperparameter_sensitivity}
Table~\ref{tab:hyperparameter_sensitivity} reports sensitivity analyses on trimodal UR-FUNNY and MOSEI, with the remaining hyperparameters fixed. 
We vary the Tucker rank \(r_m\), the projection dimension \(r_p\) and the margin \(\xi\). The Tucker rank determines the retained dimension along each mode, while \(r_p\) determines the number of entries \(r_p^M\) in the cross-moment tensor. The margin \(\xi\) sets the threshold \(1-\xi\) above which core-energy concentration is penalized.

\begin{table}[ht]
  \centering
  \caption{Sensitivity analysis: accuracy (\%)}
  \label{tab:hyperparameter_sensitivity}
  \small
  \begin{tabular}{llcc}
    \toprule
    \textit{Hyperparameter} & \textit{Value} & \textit{UR-FUNNY} & \textit{MOSEI} \\
    \midrule
    \multirow{3}{*}{Tucker rank \(r_m\)}
      & 8  & \(65.0_{\pm0.54}\) & \(69.9_{\pm0.31}\) \\
      & 16 & \(64.9_{\pm1.34}\) & \(70.0_{\pm0.47}\) \\
      & 24 & \(65.5_{\pm0.56}\) & \(70.5_{\pm0.20}\) \\
    \midrule
    \multirow{4}{*}{Projection dimension \(r_p\)}
      & 8  & \(64.8_{\pm0.59}\) & \(70.2_{\pm0.26}\) \\
      & 16 & \(64.9_{\pm0.66}\) & \(70.2_{\pm0.20}\) \\
      & 24 & \(65.5_{\pm0.56}\) & \(70.5_{\pm0.20}\) \\
      & 32 & \(65.5_{\pm0.90}\) & \(70.2_{\pm0.45}\) \\
    \midrule
    \multirow{4}{*}{Margin \(\xi\)}
      & 0.1 & \(65.2_{\pm1.15}\) & \(70.0_{\pm0.40}\) \\
      & 0.3 & \(65.2_{\pm1.05}\) & \(70.5_{\pm0.38}\) \\
      & 0.5 & \(65.5_{\pm0.56}\) & \(70.5_{\pm0.20}\) \\
      & 0.7 & \(65.2_{\pm1.00}\) & \(70.3_{\pm0.25}\) \\
    \bottomrule
  \end{tabular}
\end{table}

As shown in Table~\ref{tab:hyperparameter_sensitivity}, a Tucker rank of 24 gives the highest mean accuracy on both datasets. A projection dimension of 24 also attains the highest mean accuracy in its separate sweep, while increasing it to 32 enlarges the tensor without further mean improvement. Since the tensor is constructed before Tucker decomposition, reducing the decomposition rank does not lower the tensor construction cost. The margin sweep is comparatively stable: mean accuracy varies by 0.3 percentage points on UR-FUNNY and 0.5 on MOSEI, with \(\xi=0.5\) among the best settings on both datasets.

\subsection{Batch Size Robustness}
\label{sec:batch_size_robustness}
Since the synergy-aware regularization is computed from the empirical cross-moment tensor (Eq.~\eqref{eq:lower_dim_tensor}),
we further evaluate the sensitivity of \model to the mini-batch size.
As shown in Fig.~\ref{fig:batch_size_appendix}, \model demonstrates more pronounced robustness on datasets with larger sample sizes.
This trend is expected because larger datasets provide more reliable mini-batch estimates of cross-modal moments, whereas smaller datasets may introduce higher estimation variance.
In addition, \model achieves its best performance with a batch size of 64, and further increasing the batch size does not lead to consistent gains.
These results indicate that \model does not require excessively large batches for more reliable estimates of cross-modal moments,
while maintaining both scalability and computational efficiency.

\begin{figure}[htbp]
  \centering
  \captionsetup{skip=4pt}
  \includegraphics[width=0.6\linewidth]{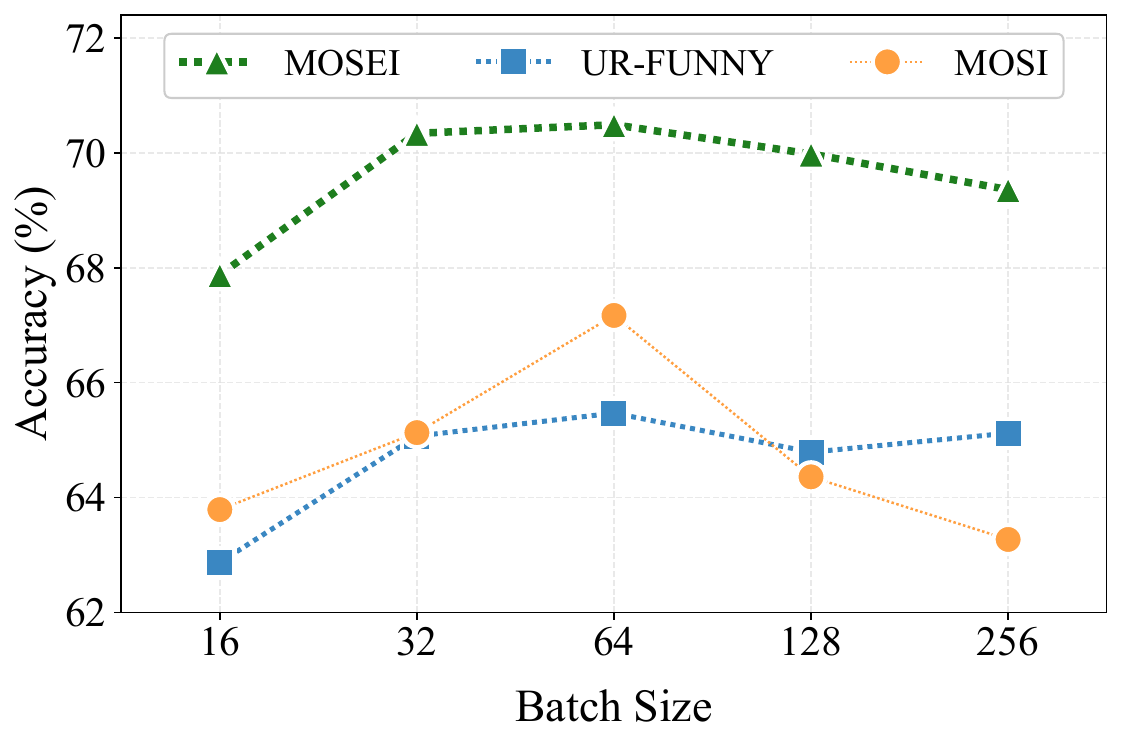}
  \caption{Performance sensitivity to mini-batch size on MOSEI, UR-FUNNY, and MOSI.}
  \label{fig:batch_size_appendix}
\end{figure}

\subsection{Experiments on MM-IMDb}
\label{sec:mmimdb_experiments}

We further evaluate \model on MM-IMDb multilabel movie-genre prediction using the original posters and plot descriptions. Table~\ref{tab:imdb_results} reports linear-probing results, where V and L denote visual and textual inputs. 
\begin{table}[ht]
  \centering
  \caption{Linear probing F1 scores (\%) on MM-IMDb. \(^\triangle\) indicates further training on unlabeled data. \(^\clubsuit\) denotes results from~\cite{dufumier2024align}.}
  \label{tab:imdb_results}
  \small
  \setlength{\tabcolsep}{3pt}
  \begin{tabular}{lccc}
    \toprule
    \textit{Model} & \textit{Modalities} & \textit{weighted-F1$\uparrow$} & \textit{macro-F1$\uparrow$} \\
    \midrule
    SimCLR\(^{\clubsuit\triangle}\)~\cite{chen2020simple} & V & \(40.35_{\pm0.23}\) & \(27.99_{\pm0.33}\) \\
    \multirow{3}{*}{CLIP\(^\clubsuit\)~\cite{radford2021learning}} & V & \(51.5\) & \(40.8\) \\
      & L & \(51.0\) & \(43.0\) \\
      & V+L & \(58.9\) & \(50.9\) \\
    SLIP\(^{\clubsuit\triangle}\)~\cite{mu2022slip} & V+L & \(56.54_{\pm0.19}\) & \(47.35_{\pm0.27}\) \\
    CLIP\(^{\clubsuit\triangle}\)~\cite{radford2021learning} & V+L & \(54.49_{\pm0.19}\) & \(44.94_{\pm0.30}\) \\
    CoMM\(_{\text{(CLIP backbone)}}\) & V+L & \(61.29_{\pm0.73}\) & \(53.79_{\pm0.22}\) \\
    \model\(_{\text{(ours, CLIP backbone)}}\) & V+L & \(\mathbf{62.30}_{\pm0.26}\) & \(\mathbf{55.44}_{\pm0.26}\) \\
    \bottomrule
  \end{tabular}
\end{table}
The CLIP results first show that combining posters and plot descriptions improves over either modality alone, indicating that the joint input provides additional predictive value. With the same backbone, \model achieves weighted and macro F1 scores of 62.30\% and 55.44\%, exceeding CoMM by 1.01 and 1.65 percentage points, respectively. 
Thus, by considering multilabel classification directly from raw image and text inputs, this experiment broadens the evaluation beyond the original MultiBench experiments and further demonstrates the effectiveness of \model.

\subsection{Experiments on Vision\&Touch}
\label{sec:vision_touch_experiments}
\begin{wraptable}{r}{0.4\textwidth}
  \centering
  \vspace{-\baselineskip}
  \caption{Linear probing accuracy (\%) on Vision\&Touch contact prediction (V\&T CP).}
  \label{tab:vision_touch_results}
  \small
  \setlength{\tabcolsep}{3pt}
  \begin{tabular}{lc}
    \toprule
    \textit{Model} & \textit{V\&T CP$\uparrow$} \\
    \midrule
    Cross & \(86.3_{\pm0.25}\) \\
    Cross+Self & \(87.6_{\pm0.26}\) \\
    CoMM & \(87.0_{\pm1.77}\) \\
    \model (ours) & \(\mathbf{88.6}_{\pm0.14}\) \\
    \bottomrule
  \end{tabular}
  \vspace{-\baselineskip}
\end{wraptable}
To evaluate the effectiveness of \model on a dataset with a larger sample size, we further conduct experiments on Vision\&Touch~\cite{lee2020making}, which contains 147,000 samples. Given visual and proprioceptive observations, the task is to predict whether contact will occur at the next time step. 

As shown in Table~\ref{tab:vision_touch_results}, adding the self-contrastive objective to Cross raises contact-prediction accuracy from 86.3\% to 87.6\%, making Cross+Self the strongest listed baseline. Our method reaches 88.6\%, improving on Cross+Self by 1.0 percentage point and on CoMM by 1.6 points. These results further demonstrate the effectiveness of \model on a substantially larger dataset for robotic contact prediction.

\subsection{Computational Complexity and Runtime Analysis}
\label{sec:computational_complexity}

\begin{table}[ht]
  \centering
  \caption{
    Training efficiency comparison between CoMM~\cite{dufumier2024align} and \model on trimodal Multibench~\cite{liang2021multibench} datasets.
    Both methods use the same multimodal backbone.
    We report the average wall-clock step time, optimizer update time, and the percentage of each step spent in backward computation.
  }
  \label{tab:complexity_analysis}
  \small
  \begin{tabular}{@{} *{5}{c} @{}}
    \toprule
    \multirow{2}{*}{\textit{Dataset}} & \multirow{2}{*}{\textit{Model}} & \multicolumn{3}{c}{\textit{Training Efficiency Metrics}} \\
    \cmidrule(lr){3-5}
     & & \textit{Step Time (ms)} & \textit{Optimizer Time (ms)} & \textit{Backward / Step Ratio (\%)} \\
    \midrule
    \multirow{2}{*}{UR-FUNNY} & CoMM & 90.73 & 3.86 & 53.82 \\
     & \model & 155.65 & 3.47 & 46.23 \\
    \midrule
    \multirow{2}{*}{MOSI} & CoMM & 90.00 & 3.90 & 54.07 \\
     & \model & 156.95 & 3.47 & 46.96 \\
    \midrule
    \multirow{2}{*}{MOSEI} & CoMM & 90.57 & 4.06 & 53.99 \\
     & \model & 159.25  & 3.63  & 47.56 \\
    \bottomrule
  \end{tabular}
\end{table}

As shown in Tab.~\ref{tab:complexity_analysis}, \model increases the per-step runtime from roughly 90 ms to 156--159 ms across the three datasets compared with CoMM~\cite{dufumier2024align}.
The optimizer time remains nearly unchanged, indicating that the overhead does not come from additional trainable parameters or a heavier parameter update.
Instead, the extra cost mainly comes from the forward and backward computation of the additional loss terms, especially the cross-moment tensor regularization in Eq.~\eqref{eq:syn_loss}.
For a mini-batch of \(B\) samples and \(M\) modalities, each projected to \(r_p\) dimensions, explicit construction of the cross-moment tensor takes \(O(B r_p^M)\) time, and the resulting tensor occupies \(O(r_p^M)\) storage before decomposition. 
With the batch size and model configuration fixed, increasing the dataset size requires more batches per epoch without changing the loss computation per batch.

These additional computations are confined to training, as the synergy-aware regularizer is not used during inference. Given the performance improvements in the evaluated settings, we regard the additional training cost as a reasonable trade-off. Potential ways to reduce this cost include TensorSketch-based Tucker approximation~\cite{NEURIPS2018_45a766fa}, which uses compressed representations in Tucker computations, and randomized tensor sparsification~\cite{DBLP:conf/sdm/Tsourakakis10}, which samples and reweights tensor entries before decomposition. Applying these approximations to our Tucker-core regularizer and evaluating their effects on training remain future work.

\section{Broader Impact}
\label{sec:broader_impact}

This work advances multimodal representation learning by explicitly preserving synergistic interactions across heterogeneous data sources. 
Its primary societal benefit lies in enabling more robust fusion frameworks for high-stakes domains such as clinical decision support, where joint modality reasoning is critical. 
Given the methodological nature of the contribution, direct societal risks are minimal; 
however, we emphasize that downstream deployment should adhere to established safeguards regarding data privacy, algorithmic fairness, and computational sustainability.

\section{Pseudo-code}
\label{sec:appendix-pseudo-code}

Algorithm~\ref{alg:hril} summarizes the training procedure of \model, using the same notation as the main paper.
Here \(D=\sum_{m=1}^{M}d_m\). We use the InfoNCE mutual-information estimator~\cite{oord2018representation} defined in the main paper:
\[
  \hat{I}_{\text{NCE}}(\vectorname{Z}, \vectorname{Z'}) = \mathbb{E}_{\substack{\vectorname{z},\vectorname{z'}_{\text{pos}}\sim p(\vectorname{Z},\vectorname{Z'})\\ \vectorname{z'}_{\text{neg}} \sim p(\vectorname{Z'})}} \left[ \log \frac{\exp \operatorname{sim}(\vectorname{z},\vectorname{z'}_{\text{pos}})}{\exp \operatorname{sim}(\vectorname{z},\vectorname{z'}_{\text{pos}}) + \sum_{\vectorname{z'}_{\text{neg}}}\exp \operatorname{sim} (\vectorname{z}, \vectorname{z'}_{\text{neg}})}\right],
\]
where \(\operatorname{sim}(\cdot, \cdot)\) denotes cosine similarity.

\begin{algorithm}[H]
\caption{\model training algorithm}
\begin{algorithmic}[1]

\Require Multimodal dataset \(\{X_1,\ldots,X_M\}\), label-preserving transformations \(\mathcal{T}^\star\), mini-batch size \(B\), unimodal encoders \(\{f_m\}_{m=1}^{M}\), shared Transformer \(g\), QR projection matrices \(\{P^{(m)}\}_{m=1}^{M}\), projection dimension \(r_p\), Tucker ranks \((r_1,\ldots,r_M)\), margin \(\xi\), alignment hyperparameters \(\gamma,\epsilon\), loss weights \(\lambda_a,\lambda_s\).
\For{each sampled mini-batch \(\{(\vectorname{x}^{(1,b)},\ldots,\vectorname{x}^{(M,b)})\}_{b=1}^{B}\)}
  \State draw \(t_1,t_2\sim \mathcal{T}^\star\)
  \For{\(b\gets 1\) to \(B\)}
    \For{\(m\gets 1\) to \(M\)}
      \State \(\vectorname{x}_{1}^{(m,b)},  \vectorname{x}_{2}^{(m,b)} \gets t_1(\vectorname{x}^{(m,b)}),  t_2(\vectorname{x}^{(m,b)})\)
\State \(\vectorname{h}_{1}^{(m,b)},\vectorname{h}_{2}^{(m,b)}\gets f_m(\vectorname{x}_{1}^{(m,b)}),f_m(\vectorname{x}_{2}^{(m,b)})\)
      \State \(\vectorname{Z}_{1,u}^{(m,b)},\vectorname{Z}_{2,u}^{(m,b)}\gets g(\vectorname{h}_{1}^{(m,b)}),g(\vectorname{h}_{2}^{(m,b)})\)
      \State \(\vectorname{Z}_{p,1}^{(m,b)}, \vectorname{Z}_{p,2}^{(m,b)} \gets \vectorname{Z}_{1,u}^{(m,b)} P^{(m)}, \vectorname{Z}_{2,u}^{(m,b)} P^{(m)}\)
    \EndFor
\State \(\vectorname{Z}_{1,f}^{(b)}\gets g(\vectorname{h}_{1}^{(1,b)},\ldots,\vectorname{h}_{1}^{(M,b)})\)
\State \(\vectorname{Z}_{2,f}^{(b)}\gets g(\vectorname{h}_{2}^{(1,b)},\ldots,\vectorname{h}_{2}^{(M,b)})\)
  \EndFor

  \For{\(m\gets 1\) to \(M\)}
    \State \(\mathcal{L}_{\text{InfoNCE}}^{(m)} \gets -\frac{1}{2}\Big(\hat{I}_{\text{NCE}}(\vectorname{Z}_{1,u}^{(m)},\vectorname{Z}_{2,f}) + \hat{I}_{\text{NCE}}(\vectorname{Z}_{2,u}^{(m)},\vectorname{Z}_{1,f})\Big)\)

  \EndFor

  \For {each view \(v\in\{1,2\}\)}
    \State \(\mathcal{T}_{r_p,v} \gets \frac{1}{B}\sum_{b=1}^{B} \vectorname{Z}_{p,v}^{(1,b)}\otimes\cdots\otimes\vectorname{Z}_{p,v}^{(M,b)}\)
    \State Tucker decompose \(\mathcal{T}_{r_p,v} \approx \mathcal{G}_v\times_1\mathbf{A}_{1,v}\cdots\times_M\mathbf{A}_{M,v}\) with ranks \((r_1,\ldots,r_M)\)
    \State \(\rho(\mathcal{G}_v) \gets \max_{i_1,\ldots,i_M} \mathcal{G}_{i_1\ldots i_M,v}^2/\|\mathcal{G}_v\|_F^2\)

    \State \(\tilde{\vectorname{Z}}_{v}^{(b)} \gets [\tilde{\vectorname{Z}}_{v}^{(1,b)};\ldots;\tilde{\vectorname{Z}}_{v}^{(M,b)}]\) for all \(b\), where \(\tilde{\vectorname{Z}}_{v}^{(m,b)} := \vectorname{Z}_{v,u}^{(m,b)}/\|\vectorname{Z}_{v,u}^{(m,b)}\|\)
    \State \(\hat{\vectorname{Z}}_{v}^{(b)} \gets \tilde{\vectorname{Z}}_{v}^{(b)}/\|\tilde{\vectorname{Z}}_{v}^{(b)}\|\) for all \(b\)
  \EndFor

  \State \(\mathcal{L}_{\text{syn}} \gets \Big([\xi - (1-\rho(\mathcal{G}_1))]^+ + [\xi - (1-\rho(\mathcal{G}_2))]^+\Big)/2\)
  \State \(\mathcal{L}_{\text{inv}} \gets \frac{1}{B}\sum_{b=1}^{B}\bigl(1-\text{sim}(\hat{\vectorname{Z}}_{1}^{(b)},\hat{\vectorname{Z}}_{2}^{(b)})\bigr)^2\)
  \State \(\bar{\vectorname{Z}} \gets (\hat{\vectorname{Z}}_{1}+\hat{\vectorname{Z}}_{2})/2\) 
  \State \(\mathcal{L}_{\text{var}} \gets \frac{1}{D}\sum_{d=1}^{D}\left[\gamma-\sqrt{\operatorname{Var}\!\left(\bar{\vectorname{Z}}_{d}\right)+\epsilon}\right]^+\)
  \State \(\mathcal{L}_{\text{align}} \gets \mathcal{L}_{\text{inv}} + \mathcal{L}_{\text{var}}\)

  \State \(\mathcal{L}_{\text{total}} \gets \sum_{m=1}^{M}\mathcal{L}_{\text{InfoNCE}}^{(m)} + \lambda_a\mathcal{L}_{\text{align}} + \lambda_s\mathcal{L}_{\text{syn}}\)
  \State update \((f_m)_{m\in[1..M]}\) and \(g\) to minimize \(\mathcal{L}_{\text{total}}\)
\EndFor
\State \textbf{return} \((f_m)_{m\in[1..M]}, g\)
\end{algorithmic}
\label{alg:hril}
\end{algorithm}

\section{Proofs}
\label{sec:all_proof}

\pairwiseInsufficientRestated*
\begin{proof}
We prove the two claims in order.

\textbf{Step 1: pairwise alignment captures at most \(\vectorname{q}_r\).}

\emph{(a) Pure synergy is inaccessible to individual modality pairs.}
For any pair \(i\neq j\) with \(M\ge3\), the set \(\{i,j\}\) is a proper subset of the modalities. Definition~\ref{def:pure_synergy} therefore gives
\[
I(\vectorname{q}_{\text{syn}};X_i,X_j\mid\vectorname{q}_r)=0.
\]
Thus, a modality pair contains no information about the pure synergistic factor once the shared factor is given.

\emph{(b) The pairwise-minimal representation excludes non-shared information.}
For \(\vectorname{Z}_{\text{pair}}^{\min}\), we use the shared information characterization in Eq.~\eqref{eq:pairwise_shared_condition}.
The encoding relation \(\vectorname{Z}=f(X)\) in the task setup ensures that a representation is a function of the multimodal observation. Conditional data processing consequently gives
\[
\begin{aligned}
0 &\le I(\vectorname{Z}_{\text{pair}}^{\min};\vectorname{q}_{1:M},\vectorname{q}_{\text{syn}}\mid\vectorname{q}_r)\\
&\le I(\vectorname{Z}_{\text{pair}}^{\min};X_{1:M}\mid\vectorname{q}_r)=0,
\end{aligned}
\]
where the last equality is Eq.~\eqref{eq:pairwise_shared_condition}. Hence,
\[
\vectorname{Z}_{\text{pair}}^{\min}\perp(\vectorname{q}_{1:M},\vectorname{q}_{\text{syn}})\mid\vectorname{q}_r.
\]

\emph{(c) Upper bound on task information.}
Assumption~\ref{Ass:syn_multi_factor} states that
\(Y\perp X_{1:M}\mid(\vectorname{q}_r,\vectorname{q}_{1:M},\vectorname{q}_{\text{syn}})\).
Since \(\vectorname{Z}_{\text{pair}}^{\min}\) is a function of \(X_{1:M}\), this also implies
\[
Y\perp\vectorname{Z}_{\text{pair}}^{\min}\mid(\vectorname{q}_r,\vectorname{q}_{1:M},\vectorname{q}_{\text{syn}}).
\]
Combining this conditional independence with part~(b), the chain rule gives
\[
\begin{aligned}
&I(\vectorname{Z}_{\text{pair}}^{\min};Y,\vectorname{q}_{1:M},\vectorname{q}_{\text{syn}}\mid\vectorname{q}_r)\\
&\quad=I(\vectorname{Z}_{\text{pair}}^{\min};\vectorname{q}_{1:M},\vectorname{q}_{\text{syn}}\mid\vectorname{q}_r)
+I(\vectorname{Z}_{\text{pair}}^{\min};Y\mid\vectorname{q}_r,\vectorname{q}_{1:M},\vectorname{q}_{\text{syn}})\\
&\quad=0.
\end{aligned}
\]
By nonnegativity of conditional mutual information, \(I(Y;\vectorname{Z}_{\text{pair}}^{\min}\mid\vectorname{q}_r)=0\). Another application of the chain rule yields
\[
\begin{aligned}
I(Y;\vectorname{Z}_{\text{pair}}^{\min})
&=I(Y;\vectorname{q}_r)
+I(Y;\vectorname{Z}_{\text{pair}}^{\min}\mid\vectorname{q}_r)
-I(Y;\vectorname{q}_r\mid\vectorname{Z}_{\text{pair}}^{\min})\\
&\le I(Y;\vectorname{q}_r).
\end{aligned}
\]

\textbf{Step 2: three-layer gap decomposition.}
By the chain rule applied to the latent factors,
\[
\begin{aligned}
I(Y;\vectorname{q}_r,\vectorname{q}_{1:M},\vectorname{q}_{\text{syn}})
&=I(Y;\vectorname{q}_r)
+I(Y;\vectorname{q}_{1:M}\mid\vectorname{q}_r)\\
&\quad+I(Y;\vectorname{q}_{\text{syn}}\mid\vectorname{q}_r,\vectorname{q}_{1:M}).
\end{aligned}
\]
Assumption~\ref{Ass:syn_multi_factor} and the condition stated immediately after it give, respectively,
\[
\begin{aligned}
I(Y;X_{1:M}\mid\vectorname{q}_r,\vectorname{q}_{1:M},\vectorname{q}_{\text{syn}})&=0,\\
I(Y;\vectorname{q}_r,\vectorname{q}_{1:M},\vectorname{q}_{\text{syn}}\mid X_{1:M})&=0.
\end{aligned}
\]
Expanding \(I(Y;X_{1:M},\vectorname{q}_r,\vectorname{q}_{1:M},\vectorname{q}_{\text{syn}})\) in the two possible orders therefore yields
\[
I(Y;X_{1:M})=I(Y;\vectorname{q}_r,\vectorname{q}_{1:M},\vectorname{q}_{\text{syn}}).
\]
Together with the bound in Step~1, this gives
\[
\begin{aligned}
I(Y;X_{1:M})-I(Y;\vectorname{Z}_{\text{pair}}^{\min})
&\ge I(Y;\vectorname{q}_{1:M}\mid\vectorname{q}_r)\\
&\quad+I(Y;\vectorname{q}_{\text{syn}}\mid\vectorname{q}_r,\vectorname{q}_{1:M}).
\end{aligned}
\]
The first term is the unique gap; the second is the synergistic gap.
The second term is strictly positive by hypothesis, so the total information gap is strictly positive.
\end{proof}

\conditionalTCRestated*

\begin{proof}
We provide the proof for \(M=3\), as the extension to general \(M\) is straightforward.

By definition,
\[
\mathrm{TC}(X_1,X_2,X_3\mid \vectorname{q}_r)
= H(X_1\mid \vectorname{q}_r)+H(X_2\mid \vectorname{q}_r)+H(X_3\mid \vectorname{q}_r)-H(X_1,X_2,X_3\mid \vectorname{q}_r).
\]
Using the chain rule,
\[
H(X_1,X_2,X_3\mid \vectorname{q}_r)
= H(X_1\mid \vectorname{q}_r)+H(X_2\mid X_1,\vectorname{q}_r)+H(X_3\mid X_1,X_2,\vectorname{q}_r).
\]
Substituting this identity yields
\[
\mathrm{TC}(X_1,X_2,X_3\mid \vectorname{q}_r)
= I(X_1;X_2\mid \vectorname{q}_r)+I(X_3;X_1,X_2\mid \vectorname{q}_r).
\]
The symmetric form follows by applying the chain rule under each permutation.
\end{proof}

\synergyTCRestated*

\begin{proof}
We first establish the above equation and further interpret this theorem.

\textbf{Step 1: the identity.}
By the definition of mutual information,
\[
I(\vectorname{q}_{\text{syn}};X_{1:M}\mid \vectorname{q}_r)
= H(X_{1:M}\mid \vectorname{q}_r) - H(X_{1:M}\mid \vectorname{q}_r,\vectorname{q}_{\text{syn}}).
\]
Recall that \(\mathrm{TC}(X_{1:M}\mid C):=\sum_{m}H(X_m\mid C)-H(X_{1:M}\mid C)\) for any conditioning variable \(C\). Rewriting both joint entropies:
\begin{align*}
H(X_{1:M}\mid \vectorname{q}_r)
&= \textstyle\sum_m H(X_m\mid \vectorname{q}_r) - \mathrm{TC}(X_{1:M}\mid \vectorname{q}_r),\\
H(X_{1:M}\mid \vectorname{q}_r,\vectorname{q}_{\text{syn}})
&= \textstyle\sum_m H(X_m\mid \vectorname{q}_r,\vectorname{q}_{\text{syn}}) - \mathrm{TC}(X_{1:M}\mid \vectorname{q}_r,\vectorname{q}_{\text{syn}}).
\end{align*}
Definition~\ref{def:pure_synergy} gives \(I(\vectorname{q}_{\text{syn}};X_m\mid\vectorname{q}_r)=0\) for every singleton \(\{m\}\subsetneq\{1,\dots,M\}\), so \(H(X_m\mid\vectorname{q}_r,\vectorname{q}_{\text{syn}})=H(X_m\mid\vectorname{q}_r)\). Substituting and cancelling:
\[
I(\vectorname{q}_{\text{syn}};X_{1:M}\mid \vectorname{q}_r)
= \mathrm{TC}(X_{1:M}\mid \vectorname{q}_r,\vectorname{q}_{\text{syn}})
  - \mathrm{TC}(X_{1:M}\mid \vectorname{q}_r).
\]

\textbf{Step 2: proper-subset marginals remain unchanged.}
Definition~\ref{def:pure_synergy} gives \(\vectorname{q}_{\text{syn}}\perp X_A\mid\vectorname{q}_r\) for every proper subset \(A\subsetneq\{1,\ldots,M\}\). Hence conditioning on \(\vectorname{q}_{\text{syn}}\) leaves all proper-subset marginals unchanged, and the TC difference comes from the change in the full joint distribution. For \(M\ge3\), this includes every pairwise marginal.
\end{proof}

\begin{proposition}[Shared-Factor Dominance Implies Redundant Low-Order Structure]
  \label{prop:redundancy_low_rank}
  Under Assumption~\ref{Ass:syn_multi_factor} and the representation model (Eq.~\eqref{eq:factor_decomposition}), 
  let \(W_{ij}^* = \arg\min_{W} \mathbb{E}\|\vectorname{Z}_i - W\vectorname{Z}_j\|^2\) be the population-optimal cross-modal linear predictor. Then:
  \begin{enumerate}
    \item[\textbf{(a)}] \textbf{(Rank bound)} \(\operatorname{rank}(W_{ij}^*) \le d_r\), where \(d_r = \dim(\vectorname{c}_r)\) is the dimension of the shared subspace.
    \item[\textbf{(b)}] \textbf{(Collapse under dominance)} When the shared-to-total energy ratio approaches one, the normalized reconstruction error vanishes and the embeddings exhibit a redundant low-order structure.
  \end{enumerate}
\end{proposition}

\begin{proof}
    \textbf{(a) Rank bound.}
    Under the model \(\vectorname{Z}_m = V_m^r \vectorname{c}_r + V_m^u \vectorname{c}_m + \boldsymbol{\epsilon}_m\) with mutually uncorrelated components, the cross-covariance is
    \[
    \operatorname{Cov}(\vectorname{Z}_i, \vectorname{Z}_j) = V_i^r \Sigma_r (V_j^r)^\top,
    \]
    where \(\Sigma_r \coloneqq \operatorname{Cov}(\vectorname{c}_r) \succ 0\). This matrix has rank at most \(\min(\operatorname{rank}(V_i^r), \operatorname{rank}(V_j^r)) \le d_r\). 
    The auto-covariance of \(\vectorname{Z}_j\) is
    \[
    \Sigma_j \coloneqq \operatorname{Cov}(\vectorname{Z}_j) = V_j^r \Sigma_r (V_j^r)^\top + V_j^u \Sigma_j^u (V_j^u)^\top + \Sigma_{\boldsymbol{\epsilon}_j} \succ 0,
    \]
    where \(\Sigma_j^u \coloneqq \operatorname{Cov}(\vectorname{c}_j)\) and \(\Sigma_{\boldsymbol{\epsilon}_j} \coloneqq \operatorname{Cov}(\boldsymbol{\epsilon}_j)\). Positivity follows from the noise term. 
    The population-optimal predictor is
    \[
    W_{ij}^* = V_i^r \Sigma_r (V_j^r)^\top \Sigma_j^{-1},
    \]
    whose rank is at most \(d_r\) since \(\operatorname{col}(W_{ij}^*) \subseteq \operatorname{col}(V_i^r)\).

    \textbf{(b) Collapse under dominance.}
    For simplicity, we assume that the observed vectors have been centered and normalized sample-wise:
    \(\mathbb{E}[\mathbf{Z}_m] = \mathbf{0}\) and \(\mathbb{E}\|\mathbf{Z}_m\|^2 = \operatorname{tr}(\Sigma_m) = 1\).
    Define the unique-plus-noise component of modality \(m\) as
    \[
    \mathbf{n}_m \coloneqq V_m^u \mathbf{c}_m + \boldsymbol{\epsilon}_m, \qquad
    \Sigma_{\mathbf{n}_m} \coloneqq \operatorname{Cov}(\mathbf{n}_m)
    = V_m^u \Sigma_m^u (V_m^u)^\top + \Sigma_{\boldsymbol{\epsilon}_m} \succeq 0.
    \]
    Then \(\mathbf{Z}_m = V_m^r \mathbf{c}_r + \mathbf{n}_m\) with \(\operatorname{Cov}(\mathbf{c}_r, \mathbf{n}_m) = \mathbf{0}\).
    We impose the following regularity conditions:
    \begin{enumerate}
        \item[(A1)] \(\Sigma_{\mathbf{n}_j} \succ 0\) (or at least \(\Sigma_j \succ 0\)) to ensure the existence of the optimal linear predictor.
        \item[(A2)] The unique-plus-noise components of different modalities are mutually uncorrelated: \(\operatorname{Cov}(\mathbf{n}_i, \mathbf{n}_j) = \mathbf{0}\) for \(i \neq j\).
        \item[(A3)] \(V_j^r\) has full column rank \(d_r\) and \(\Sigma_r \succ 0\).
    \end{enumerate}
    Because \(\operatorname{tr}(\Sigma_m)=1\), the \emph{shared energy ratio} simplifies to
    \[
    \alpha_m \coloneqq \operatorname{tr}\bigl(V_m^r\Sigma_r(V_m^r)^\top\bigr) \in [0,1],
    \qquad\text{and}\qquad
    \operatorname{tr}(\Sigma_{\mathbf{n}_m}) = 1 - \alpha_m.
    \]
    Hence \(\alpha_m \to 1\) implies \(\operatorname{tr}(\Sigma_{\mathbf{n}_m}) \to 0\).  Since \(\Sigma_{\mathbf{n}_m} \succeq 0\),
    \[
    \|\Sigma_{\mathbf{n}_m}\|_2 \le \operatorname{tr}(\Sigma_{\mathbf{n}_m}) \to 0.
    \]

    The optimal linear predictor of \(\mathbf{c}_r\) given \(\mathbf{Z}_j\) is
    \[
    \widehat{\mathbf{c}}_r = \Sigma_r (V_j^r)^\top \Sigma_j^{-1} \mathbf{Z}_j,\qquad
    \Omega_j \coloneqq \operatorname{Cov}(\mathbf{c}_r - \widehat{\mathbf{c}}_r)
    = \Sigma_r - \Sigma_r (V_j^r)^\top \Sigma_j^{-1} V_j^r \Sigma_r,
    \]
    where \(\Sigma_j = V_j^r\Sigma_r(V_j^r)^\top + \Sigma_{\mathbf{n}_j} \succ 0\).

    To prove \(\Omega_j \to \mathbf{0}\) rigorously, we bound it by a feasible sub-optimal estimator.
    By (A3), let \(L_j = \bigl((V_j^r)^\top V_j^r\bigr)^{-1}(V_j^r)^\top\) be a left inverse of \(V_j^r\).
    Define \(\check{\mathbf{c}}_r = L_j \mathbf{Z}_j\).  Then
    \[
    \mathbf{c}_r - \check{\mathbf{c}}_r = -L_j \mathbf{n}_j,\qquad
    \operatorname{Cov}(\mathbf{c}_r - \check{\mathbf{c}}_r) = L_j \Sigma_{\mathbf{n}_j} L_j^\top.
    \]
    By the optimality of the MMSE predictor, \(\Omega_j \preceq \operatorname{Cov}(\mathbf{c}_r - \check{\mathbf{c}}_r)\).
    Taking spectral norms,
    \[
    \|\Omega_j\|_2 \le \|L_j\|_2^2 \, \|\Sigma_{\mathbf{n}_j}\|_2 \to 0,
    \]
    so \(\Omega_j \to \mathbf{0}\).

    The best linear predictor of \(\mathbf{Z}_i\) from \(\mathbf{Z}_j\) is
    \[
    W_{ij}^* = \operatorname{Cov}(\mathbf{Z}_i, \mathbf{Z}_j)\,\Sigma_j^{-1}
    = V_i^r \Sigma_r (V_j^r)^\top \Sigma_j^{-1},
    \]
    whose column space lies in \(\operatorname{col}(V_i^r)\), hence \(\operatorname{rank}(W_{ij}^*) \le d_r\).
    Its normalized mean-squared error can be written as
    \[
    \frac{\mathbb{E}\|\mathbf{Z}_i - W_{ij}^* \mathbf{Z}_j\|^2}{\mathbb{E}\|\mathbf{Z}_i\|^2}
    = 1 - \frac{\operatorname{tr}\!\bigl(V_i^r \Sigma_r (V_j^r)^\top \Sigma_j^{-1} V_j^r \Sigma_r (V_i^r)^\top\bigr)}{\operatorname{tr}(\Sigma_i)}.
    \]
    Using \(\Sigma_r (V_j^r)^\top \Sigma_j^{-1} V_j^r \Sigma_r = \Sigma_r - \Omega_j\),
    \[
    \operatorname{tr}\!\bigl(V_i^r \Sigma_r (V_j^r)^\top \Sigma_j^{-1} V_j^r \Sigma_r (V_i^r)^\top\bigr)
    = \operatorname{tr}\!\bigl(V_i^r \Sigma_r (V_i^r)^\top\bigr)
    - \operatorname{tr}\!\bigl(V_i^r \Omega_j (V_i^r)^\top\bigr)
    = \alpha_i - \operatorname{tr}\!\bigl(V_i^r \Omega_j (V_i^r)^\top\bigr).
    \]
    The second term satisfies
    \[
    0 \le \operatorname{tr}\!\bigl(V_i^r \Omega_j (V_i^r)^\top\bigr)
    \le \|V_i^r\|_2^2\, \|\Omega_j\|_2 \operatorname{rank}(V_i^r) \to 0.
    \]
    Because the normalization guarantees \(\operatorname{tr}(\Sigma_i) = 1\), we obtain
    \[
    \frac{\mathbb{E}\|\mathbf{Z}_i - W_{ij}^* \mathbf{Z}_j\|^2}{\mathbb{E}\|\mathbf{Z}_i\|^2}
    = 1 - \alpha_i + o(1).
    \]
    If \(\alpha_i \to 1\) and \(\alpha_j \to 1\) (the latter ensures \(\Omega_j\to\mathbf{0}\)), the right-hand side tends to \(0\).

\end{proof}

\lowRankRestated*

\begin{proof}
Without loss of generality, assume modality 1 satisfies the redundancy condition. 
By Definition~\ref{def:infor_red_def}, there exist linear maps \(\mathbf{W}_{1j}\) such that
\[
\mathbf{Z}_1 = \sum_{j=2}^{M} \mathbf{W}_{1j} \mathbf{Z}_j + \boldsymbol{\delta}_1, 
\quad \text{with } \operatorname{rank}(\mathbf{W}_{1j}) = r_{1j} \ll d_j \text{ and } \mathbb{E}\|\boldsymbol{\delta}_1\|^2 \to 0.
\]
Substituting into \(\mathcal{T} = \mathbb{E}[\mathbf{Z}_1 \otimes \mathbf{Z}_2 \otimes \cdots \otimes \mathbf{Z}_M]\) and taking the mode-1 unfolding yields
\[
\mathcal{T}_{(1)} = \tilde{\mathcal{T}}_{(1)} + \mathcal{E}_{(1)}, \quad \text{where } 
\tilde{\mathcal{T}}_{(1)} = \sum_{j=2}^{M} \mathbf{W}_{1j} \, \mathbb{E}\!\left[ \mathbf{Z}_j (\mathbf{Z}_2 \otimes \cdots \otimes \mathbf{Z}_M)^\top \right].
\]
By construction, \(\operatorname{rank}(\tilde{\mathcal{T}}_{(1)}) \leq \sum_{j=2}^{M} r_{1j}\); let \(R = \sum_{j=2}^{M} r_{1j}\). For the modality-wise normalized embeddings, Cauchy--Schwarz gives \(\|\mathcal{E}_{(1)}\|_F \leq (\mathbb{E}\|\boldsymbol{\delta}_1\|^2)^{1/2}\to0\). For any tolerance \(\epsilon > \|\mathcal{E}_{(1)}\|_F\), \(\tilde{\mathcal{T}}_{(1)}\) is a feasible rank-\(R\) approximation to \(\mathcal{T}_{(1)}\). By the Eckart--Young--Mirsky theorem~\cite{eckart1936approximation,mirsky1960symmetric}, the optimal rank-\(R\) approximation error satisfies
\[
\min_{\operatorname{rank}(A) \leq R} \|\mathcal{T}_{(1)} - A\|_F \leq \|\mathcal{E}_{(1)}\|_F < \epsilon.
\]
Consequently, at most \(R\) singular values of \(\mathcal{T}_{(1)}\) can exceed \(\epsilon\), \ie,
\[
\operatorname{rank}_{(1)}^{(\epsilon)}(\mathcal{T}) \leq R.
\]

Consider the full orthogonal Tucker decomposition of \(\mathcal{T}\) obtained via the Higher-Order Singular Value Decomposition (HOSVD):
\[
\mathcal{T} = \mathcal{G}^{\text{full}} \times_1 Q^{(1)} \times_2 \cdots \times_M Q^{(M)},
\]
where each \(Q^{(m)} \in \mathbb{R}^{d_m \times d_m}\) is an orthogonal matrix. The mode-1 unfolding gives
\[
\mathcal{T}_{(1)} = Q^{(1)} \mathcal{G}^{\text{full}}_{(1)} (Q^{(M)} \otimes \cdots \otimes Q^{(2)})^\top,
\]
and by orthogonality, \(\sigma_k(\mathcal{T}_{(1)}) = \sigma_k(\mathcal{G}^{\text{full}}_{(1)})\) for all \(k\). Hence,
\[
\operatorname{rank}_{(1)}^{(\epsilon)}(\mathcal{G}^{\text{full}}) = \operatorname{rank}_{(1)}^{(\epsilon)}(\mathcal{T}) \leq R.
\]

The Gaussian-QR projection in Eq.~\eqref{eq:lower_dim_tensor} and the orthogonal Tucker factors used in Appendix~\ref{sec:exp_details} cannot increase the rank or residual norm of this approximation. Hence the truncated core also satisfies \(\operatorname{rank}_{(1)}^{(\epsilon)}(\mathcal G)\le R\).

\end{proof}

\energyConcentrationRestated*

\begin{proof}
For \(\mathcal G\ne0\), let \((i_1^\star,\dots,i_M^\star)\) be an index attaining the maximum entry magnitude of \(\mathcal{G}\), and define a one-sparse tensor \(\mathcal{G}^{(1)}\) that keeps only \(\mathcal{G}_{i_1^\star,\dots,i_M^\star}\). Then
\[
\|\mathcal{G}-\mathcal{G}^{(1)}\|_F^2
= \|\mathcal{G}\|_F^2 - \max_{i_1,\dots,i_M}\mathcal{G}_{i_1,\dots,i_M}^2.
\]
If \(\rho(\mathcal{G}) \ge 1-\delta\), then
\[
\frac{\|\mathcal{G}-\mathcal{G}^{(1)}\|_F^2}{\|\mathcal{G}\|_F^2}
\le \delta,
\]
and hence
\[
\frac{\|\mathcal{G}-\mathcal{G}^{(1)}\|_F}{\|\mathcal{G}\|_F}
\le \sqrt{\delta}.
\]
Therefore \(\mathcal{G}\) is close to a one-sparse tensor in relative Frobenius norm, which indicates a degenerate structure dominated by a single separable component when \(\delta\) is small.
\end{proof}

\begin{lemma}[Johnson--Lindenstrauss lemma~\cite{dasgupta2003elementary}]
    If the columns of \(P^{(m)}\) are i.i.d.\ standard Gaussian vectors that have been orthonormalized, then for any fixed finite sample set \(\{Z^{(m)}_b\}_{b=1}^B\) independent of \(P^{(m)}\), by choosing \(r_p = O(\epsilon^{-2}\log B)\), the following holds with high probability:
    \[
\begin{aligned}
    (1-\epsilon)\frac{r_p}{d_m}\|Z^{(m)}_b - Z^{(m)}_{b'}\|_2^2
    &\le \|(Z^{(m)}_b-Z^{(m)}_{b'})P^{(m)}\|_2^2\\
    &\le (1+\epsilon)\frac{r_p}{d_m}\|Z^{(m)}_b - Z^{(m)}_{b'}\|_2^2.
    \end{aligned}
    \]
    In other words, pairwise distances are approximately preserved up to a common scale under the projection.
    \label{lemma:JL}
\end{lemma}

\begin{proof}
  The lemma follows from~\cite{dasgupta2003elementary},
and we do not provide the proof here for brevity.
\end{proof}

\newpage
\providecommand{\modify}[1]{\textcolor{blue}{#1}}

\section*{NeurIPS Paper Checklist}

\begin{enumerate}

\item {\bf Claims}
    \item[] Question: Do the main claims made in the abstract and introduction accurately reflect the paper's contributions and scope?
    \item[] Answer: \answerYes{} 
    \item[] Justification: We focus on the synergistic information in self-supervised multimodal representation learning. We first characterizes the synergistic gap for more than two modalities, and then propose \model to preserve the information capacity of synergy. Extensive experiments on diverse real-world datasets verify the effectiveness of our method.
    \item[] Guidelines:
    \begin{itemize}
        \item The answer \answerNA{} means that the abstract and introduction do not include the claims made in the paper.
        \item The abstract and/or introduction should clearly state the claims made, including the contributions made in the paper and important assumptions and limitations. A \answerNo{} or \answerNA{} answer to this question will not be perceived well by the reviewers. 
        \item The claims made should match theoretical and experimental results, and reflect how much the results can be expected to generalize to other settings. 
        \item It is fine to include aspirational goals as motivation as long as it is clear that these goals are not attained by the paper. 
    \end{itemize}

\item {\bf Limitations}
    \item[] Question: Does the paper discuss the limitations of the work performed by the authors?
    \item[] Answer: \answerYes{} 
    \item[] Justification: As shown in Appendix~\ref{sec:limitations_future_work}, we discuss the limitations of this work.
    \item[] Guidelines:
    \begin{itemize}
        \item The answer \answerNA{} means that the paper has no limitation while the answer \answerNo{} means that the paper has limitations, but those are not discussed in the paper. 
        \item The authors are encouraged to create a separate ``Limitations'' section in their paper.
        \item The paper should point out any strong assumptions and how robust the results are to violations of these assumptions (e.g., independence assumptions, noiseless settings, model well-specification, asymptotic approximations only holding locally). The authors should reflect on how these assumptions might be violated in practice and what the implications would be.
        \item The authors should reflect on the scope of the claims made, e.g., if the approach was only tested on a few datasets or with a few runs. In general, empirical results often depend on implicit assumptions, which should be articulated.
        \item The authors should reflect on the factors that influence the performance of the approach. For example, a facial recognition algorithm may perform poorly when image resolution is low or images are taken in low lighting. Or a speech-to-text system might not be used reliably to provide closed captions for online lectures because it fails to handle technical jargon.
        \item The authors should discuss the computational efficiency of the proposed algorithms and how they scale with dataset size.
        \item If applicable, the authors should discuss possible limitations of their approach to address problems of privacy and fairness.
        \item While the authors might fear that complete honesty about limitations might be used by reviewers as grounds for rejection, a worse outcome might be that reviewers discover limitations that aren't acknowledged in the paper. The authors should use their best judgment and recognize that individual actions in favor of transparency play an important role in developing norms that preserve the integrity of the community. Reviewers will be specifically instructed to not penalize honesty concerning limitations.
    \end{itemize}

\item {\bf Theory assumptions and proofs}
    \item[] Question: For each theoretical result, does the paper provide the full set of assumptions and a complete (and correct) proof?
    \item[] Answer: \answerYes{} 
    \item[] Justification: All theorems, formulas, propositions, and proofs in the main text are properly numbered and cross-referenced. The key proofs have been included in the appendix~\ref{sec:all_proof}.
    \item[] Guidelines:
    \begin{itemize}
        \item The answer \answerNA{} means that the paper does not include theoretical results. 
        \item All the theorems, formulas, and proofs in the paper should be numbered and cross-referenced.
        \item All assumptions should be clearly stated or referenced in the statement of any theorems.
        \item The proofs can either appear in the main paper or the supplemental material, but if they appear in the supplemental material, the authors are encouraged to provide a short proof sketch to provide intuition. 
        \item Inversely, any informal proof provided in the core of the paper should be complemented by formal proofs provided in appendix or supplemental material.
        \item Theorems and Lemmas that the proof relies upon should be properly referenced. 
    \end{itemize}

    \item {\bf Experimental result reproducibility}
    \item[] Question: Does the paper fully disclose all the information needed to reproduce the main experimental results of the paper to the extent that it affects the main claims and/or conclusions of the paper (regardless of whether the code and data are provided or not)?
    \item[] Answer: \answerYes{} 
    \item[] Justification: Detailed analysis and results can be found in Sec.~\ref{sec:exper}. 
    Experimental details can be found in Appendix~\ref{sec:exp_details}.
    After the anonymity period, we will open-source our training code and key hyperparameters for reproducing.
    \item[] Guidelines:
    \begin{itemize}
        \item The answer \answerNA{} means that the paper does not include experiments.
        \item If the paper includes experiments, a \answerNo{} answer to this question will not be perceived well by the reviewers: Making the paper reproducible is important, regardless of whether the code and data are provided or not.
        \item If the contribution is a dataset and\slash or model, the authors should describe the steps taken to make their results reproducible or verifiable. 
        \item Depending on the contribution, reproducibility can be accomplished in various ways. For example, if the contribution is a novel architecture, describing the architecture fully might suffice, or if the contribution is a specific model and empirical evaluation, it may be necessary to either make it possible for others to replicate the model with the same dataset, or provide access to the model. In general. releasing code and data is often one good way to accomplish this, but reproducibility can also be provided via detailed instructions for how to replicate the results, access to a hosted model (e.g., in the case of a large language model), releasing of a model checkpoint, or other means that are appropriate to the research performed.
        \item While NeurIPS does not require releasing code, the conference does require all submissions to provide some reasonable avenue for reproducibility, which may depend on the nature of the contribution. For example
        \begin{enumerate}
            \item If the contribution is primarily a new algorithm, the paper should make it clear how to reproduce that algorithm.
            \item If the contribution is primarily a new model architecture, the paper should describe the architecture clearly and fully.
            \item If the contribution is a new model (e.g., a large language model), then there should either be a way to access this model for reproducing the results or a way to reproduce the model (e.g., with an open-source dataset or instructions for how to construct the dataset).
            \item We recognize that reproducibility may be tricky in some cases, in which case authors are welcome to describe the particular way they provide for reproducibility. In the case of closed-source models, it may be that access to the model is limited in some way (e.g., to registered users), but it should be possible for other researchers to have some path to reproducing or verifying the results.
        \end{enumerate}
    \end{itemize}

\item {\bf Open access to data and code}
    \item[] Question: Does the paper provide open access to the data and code, with sufficient instructions to faithfully reproduce the main experimental results, as described in supplemental material?
    \item[] Answer: \answerYes{} 
    \item[] Justification: After the anonymity period, we will open-source our training code and key hyperparameters for reproducing.
    \item[] Guidelines:
    \begin{itemize}
        \item The answer \answerNA{} means that paper does not include experiments requiring code.
        \item Please see the NeurIPS code and data submission guidelines (\url{https://neurips.cc/public/guides/CodeSubmissionPolicy}) for more details.
        \item While we encourage the release of code and data, we understand that this might not be possible, so \answerNo{} is an acceptable answer. Papers cannot be rejected simply for not including code, unless this is central to the contribution (e.g., for a new open-source benchmark).
        \item The instructions should contain the exact command and environment needed to run to reproduce the results. See the NeurIPS code and data submission guidelines (\url{https://neurips.cc/public/guides/CodeSubmissionPolicy}) for more details.
        \item The authors should provide instructions on data access and preparation, including how to access the raw data, preprocessed data, intermediate data, and generated data, etc.
        \item The authors should provide scripts to reproduce all experimental results for the new proposed method and baselines. If only a subset of experiments are reproducible, they should state which ones are omitted from the script and why.
        \item At submission time, to preserve anonymity, the authors should release anonymized versions (if applicable).
        \item Providing as much information as possible in supplemental material (appended to the paper) is recommended, but including URLs to data and code is permitted.
    \end{itemize}

\item {\bf Experimental setting/details}
    \item[] Question: Does the paper specify all the training and test details (e.g., data splits, hyperparameters, how they were chosen, type of optimizer) necessary to understand the results?
    \item[] Answer: \answerYes{} 
    \item[] Justification: As shown in Appendix~\ref{sec:exp_details}, we provide comprehensive details of training procedures. Detailed implementation code will be released upon acceptance.
    \item[] Guidelines:
    \begin{itemize}
        \item The answer \answerNA{} means that the paper does not include experiments.
        \item The experimental setting should be presented in the core of the paper to a level of detail that is necessary to appreciate the results and make sense of them.
        \item The full details can be provided either with the code, in appendix, or as supplemental material.
    \end{itemize}

\item {\bf Experiment statistical significance}
    \item[] Question: Does the paper report error bars suitably and correctly defined or other appropriate information about the statistical significance of the experiments?
    \item[] Answer: \answerYes{} 
    \item[] Justification: For all experiments, we report the mean and standard deviation across five independent runs with random seeds in the range \([42, 46]\).
    \item[] Guidelines:
    \begin{itemize}
        \item The answer \answerNA{} means that the paper does not include experiments.
        \item The authors should answer \answerYes{} if the results are accompanied by error bars, confidence intervals, or statistical significance tests, at least for the experiments that support the main claims of the paper.
        \item The factors of variability that the error bars are capturing should be clearly stated (for example, train/test split, initialization, random drawing of some parameter, or overall run with given experimental conditions).
        \item The method for calculating the error bars should be explained (closed form formula, call to a library function, bootstrap, etc.)
        \item The assumptions made should be given (e.g., Normally distributed errors).
        \item It should be clear whether the error bar is the standard deviation or the standard error of the mean.
        \item It is OK to report 1-sigma error bars, but one should state it. The authors should preferably report a 2-sigma error bar than state that they have a 96\% CI, if the hypothesis of Normality of errors is not verified.
        \item For asymmetric distributions, the authors should be careful not to show in tables or figures symmetric error bars that would yield results that are out of range (e.g., negative error rates).
        \item If error bars are reported in tables or plots, the authors should explain in the text how they were calculated and reference the corresponding figures or tables in the text.
    \end{itemize}

\item {\bf Experiments compute resources}
    \item[] Question: For each experiment, does the paper provide sufficient information on the computer resources (type of compute workers, memory, time of execution) needed to reproduce the experiments?
    \item[] Answer: \answerYes{} 
    \item[] Justification: We use a single NVIDIA RTX 4090 GPU with 24GB of VRAM for all experiments. More details can be found in Appendix~\ref{sec:exp_details}.
    \item[] Guidelines:
    \begin{itemize}
        \item The answer \answerNA{} means that the paper does not include experiments.
        \item The paper should indicate the type of compute workers CPU or GPU, internal cluster, or cloud provider, including relevant memory and storage.
        \item The paper should provide the amount of compute required for each of the individual experimental runs as well as estimate the total compute. 
        \item The paper should disclose whether the full research project required more compute than the experiments reported in the paper (e.g., preliminary or failed experiments that didn't make it into the paper). 
    \end{itemize}
    
\item {\bf Code of ethics}
    \item[] Question: Does the research conducted in the paper conform, in every respect, with the NeurIPS Code of Ethics \url{https://neurips.cc/public/EthicsGuidelines}?
    \item[] Answer: \answerYes{} 
    \item[] Justification: All experiments utilize publicly available benchmark datasets in accordance with their respective licensing terms. Since this work involves neither human subjects nor personally identifiable information, institutional review board (IRB) approval is not required. Comprehensive experimental configurations are documented in Appendix~\ref{sec:exp_details}.
    \item[] Guidelines:
    \begin{itemize}
        \item The answer \answerNA{} means that the authors have not reviewed the NeurIPS Code of Ethics.
        \item If the authors answer \answerNo, they should explain the special circumstances that require a deviation from the Code of Ethics.
        \item The authors should make sure to preserve anonymity (e.g., if there is a special consideration due to laws or regulations in their jurisdiction).
    \end{itemize}

\item {\bf Broader impacts}
    \item[] Question: Does the paper discuss both potential positive societal impacts and negative societal impacts of the work performed?
    \item[] Answer: \answerYes{} 
    \item[] Justification: Please refer to Appendix~\ref{sec:broader_impact}.
    \item[] Guidelines:
    \begin{itemize}
        \item The answer \answerNA{} means that there is no societal impact of the work performed.
        \item If the authors answer \answerNA{} or \answerNo, they should explain why their work has no societal impact or why the paper does not address societal impact.
        \item Examples of negative societal impacts include potential malicious or unintended uses (e.g., disinformation, generating fake profiles, surveillance), fairness considerations (e.g., deployment of technologies that could make decisions that unfairly impact specific groups), privacy considerations, and security considerations.
        \item The conference expects that many papers will be foundational research and not tied to particular applications, let alone deployments. However, if there is a direct path to any negative applications, the authors should point it out. For example, it is legitimate to point out that an improvement in the quality of generative models could be used to generate Deepfakes for disinformation. On the other hand, it is not needed to point out that a generic algorithm for optimizing neural networks could enable people to train models that generate Deepfakes faster.
        \item The authors should consider possible harms that could arise when the technology is being used as intended and functioning correctly, harms that could arise when the technology is being used as intended but gives incorrect results, and harms following from (intentional or unintentional) misuse of the technology.
        \item If there are negative societal impacts, the authors could also discuss possible mitigation strategies (e.g., gated release of models, providing defenses in addition to attacks, mechanisms for monitoring misuse, mechanisms to monitor how a system learns from feedback over time, improving the efficiency and accessibility of ML).
    \end{itemize}
    
\item {\bf Safeguards}
    \item[] Question: Does the paper describe safeguards that have been put in place for responsible release of data or models that have a high risk for misuse (e.g., pre-trained language models, image generators, or scraped datasets)?
    \item[] Answer: \answerNA{} 
    \item[] Justification: We utilize the pre-processed datasets provided by MultiBench~\cite{liang2021multibench}, which have been anonymized to safeguard personal privacy. 
    \item[] Guidelines:
    \begin{itemize}
        \item The answer \answerNA{} means that the paper poses no such risks.
        \item Released models that have a high risk for misuse or dual-use should be released with necessary safeguards to allow for controlled use of the model, for example by requiring that users adhere to usage guidelines or restrictions to access the model or implementing safety filters. 
        \item Datasets that have been scraped from the Internet could pose safety risks. The authors should describe how they avoided releasing unsafe images.
        \item We recognize that providing effective safeguards is challenging, and many papers do not require this, but we encourage authors to take this into account and make a best faith effort.
    \end{itemize}

\item {\bf Licenses for existing assets}
    \item[] Question: Are the creators or original owners of assets (e.g., code, data, models), used in the paper, properly credited and are the license and terms of use explicitly mentioned and properly respected?
    \item[] Answer: \answerYes{} 
    \item[] Justification: We have provided detailed descriptions and clearly marked the sources and citations for all models and frameworks involved in the paper within the experimental section. 
    For open-source code, we have included comprehensive comments and explanations for all imported packages and foundational code used.
    \item[] Guidelines:
    \begin{itemize}
        \item The answer \answerNA{} means that the paper does not use existing assets.
        \item The authors should cite the original paper that produced the code package or dataset.
        \item The authors should state which version of the asset is used and, if possible, include a URL.
        \item The name of the license (e.g., CC-BY 4.0) should be included for each asset.
        \item For scraped data from a particular source (e.g., website), the copyright and terms of service of that source should be provided.
        \item If assets are released, the license, copyright information, and terms of use in the package should be provided. For popular datasets, \url{paperswithcode.com/datasets} has curated licenses for some datasets. Their licensing guide can help determine the license of a dataset.
        \item For existing datasets that are re-packaged, both the original license and the license of the derived asset (if it has changed) should be provided.
        \item If this information is not available online, the authors are encouraged to reach out to the asset's creators.
    \end{itemize}

\item {\bf New assets}
    \item[] Question: Are new assets introduced in the paper well documented and is the documentation provided alongside the assets?
    \item[] Answer: \answerYes{} 
    \item[] Justification:  This work introduces new assets, including the training code and dataset developed for the \model representation learning framework. These resources are comprehensively documented in the experimental section, which details the model architecture, training protocols, and dataset specifications, alongside complete citations for all underlying frameworks and libraries. The accompanying open-source codebase features extensive inline annotations that clarify the functionality of external dependencies and the organization of core modules. A standardized documentation template is provided alongside the release, specifying training configurations, confirming dataset licensing compliance, and outlining model limitations.
    \item[] Guidelines:
    \begin{itemize}
        \item The answer \answerNA{} means that the paper does not release new assets.
        \item Researchers should communicate the details of the dataset\slash code\slash model as part of their submissions via structured templates. This includes details about training, license, limitations, etc. 
        \item The paper should discuss whether and how consent was obtained from people whose asset is used.
        \item At submission time, remember to anonymize your assets (if applicable). You can either create an anonymized URL or include an anonymized zip file.
    \end{itemize}

\item {\bf Crowdsourcing and research with human subjects}
    \item[] Question: For crowdsourcing experiments and research with human subjects, does the paper include the full text of instructions given to participants and screenshots, if applicable, as well as details about compensation (if any)? 
    \item[] Answer: \answerNA{} 
    \item[] Justification: This work does not involve crowdsourcing or research with human subjects.
    \item[] Guidelines:
    \begin{itemize}
        \item The answer \answerNA{} means that the paper does not involve crowdsourcing nor research with human subjects.
        \item Including this information in the supplemental material is fine, but if the main contribution of the paper involves human subjects, then as much detail as possible should be included in the main paper. 
        \item According to the NeurIPS Code of Ethics, workers involved in data collection, curation, or other labor should be paid at least the minimum wage in the country of the data collector. 
    \end{itemize}

\item {\bf Institutional review board (IRB) approvals or equivalent for research with human subjects}
    \item[] Question: Does the paper describe potential risks incurred by study participants, whether such risks were disclosed to the subjects, and whether Institutional Review Board (IRB) approvals (or an equivalent approval/review based on the requirements of your country or institution) were obtained?
    \item[] Answer: \answerNA{} 
    \item[] Justification: This work does not involve crowdsourcing or research with human subjects.
    \item[] Guidelines:
    \begin{itemize}
        \item The answer \answerNA{} means that the paper does not involve crowdsourcing nor research with human subjects.
        \item Depending on the country in which research is conducted, IRB approval (or equivalent) may be required for any human subjects research. If you obtained IRB approval, you should clearly state this in the paper. 
        \item We recognize that the procedures for this may vary significantly between institutions and locations, and we expect authors to adhere to the NeurIPS Code of Ethics and the guidelines for their institution. 
        \item For initial submissions, do not include any information that would break anonymity (if applicable), such as the institution conducting the review.
    \end{itemize}

\item {\bf Declaration of LLM usage}
    \item[] Question: Does the paper describe the usage of LLMs if it is an important, original, or non-standard component of the core methods in this research? Note that if the LLM is used only for writing, editing, or formatting purposes and does \emph{not} impact the core methodology, scientific rigor, or originality of the research, declaration is not required.
    \item[] Answer: \answerNA{} 
    \item[] Justification: In this work, the LLM is utilized solely for writing and editing purposes. 
    \item[] Guidelines:
    \begin{itemize}
        \item The answer \answerNA{} means that the core method development in this research does not involve LLMs as any important, original, or non-standard components.
        \item Please refer to our LLM policy in the NeurIPS handbook for what should or should not be described.
    \end{itemize}

\end{enumerate}

\end{document}